\documentclass[11pt,draftcls,onecolumn]{IEEEtran}
\usepackage{amsfonts}
\usepackage{amssymb}
\usepackage{stfloats}
\usepackage{cite}
\usepackage{graphicx}
\usepackage{psfrag}
\usepackage{subfigure}
\usepackage{amsmath}
\usepackage{array}
\usepackage{fancyhdr}
\usepackage{color}
\newtheorem{Definition}{Definition}
\newtheorem{Lemma}{Lemma}
\newtheorem{Problem}{Problem}
\newtheorem{Algorithm}{Algorithm}

\newtheorem{Theorem}{Theorem}
\newtheorem{Remark}{Remark}
\newtheorem{Example}{Example}
\newtheorem{Assumption}{Assumption}

\usepackage[margin=2.45cm]{geometry}
\newcommand{\refbrk}[1]{(\ref{#1})}

\title{\huge{Queue-Aware Distributive Resource Control for Delay-Sensitive Two-Hop MIMO Cooperative Systems}}

\author{Rui Wang, Vincent K. N. Lau and Ying Cui\\
Department of ECE,
The Hong Kong University of Science and Technology\\
Email:  wray@ust.hk, eeknlau@ust.hk, cuiying@ust.hk }

\begin{document}

\maketitle

\begin{abstract}
In this paper, we consider a queue-aware distributive resource
control algorithm for two-hop MIMO cooperative systems. We shall
illustrate that relay buffering is an effective way to reduce the
intrinsic half-duplex penalty in cooperative systems. The complex
interactions of the queues at the source node and the relays are
modeled as an average-cost infinite horizon Markov Decision Process
(MDP). The traditional approach solving this MDP problem involves
centralized control with huge complexity. To obtain a distributive
and low complexity solution, we introduce a linear structure which
approximates the value function of the associated Bellman equation
by the sum of per-node value functions. We derive a distributive
{\em two-stage two-winner auction-based} control policy which is a
function of the local CSI and local QSI only. Furthermore, to
estimate the {\em best fit} approximation parameter, we propose a
distributive online stochastic learning algorithm using stochastic
approximation theory. Finally, we establish technical conditions for
almost-sure convergence and show that under heavy traffic, the
proposed low complexity distributive control is global optimal.
\end{abstract}


%
%
%
%
%

\section{Introduction}

Cooperative relay communication has been a hot research topic in
both the academia \cite{Meulen:68,Cover:79} and the industry
\cite{WiMaxRelay:site,WINNER:site} because it could exploit the
broadcast nature of wireless communication to achieve cooperative
diversity. 
\textcolor{black}{One potential issue of cooperative communication
is the half-duplex penalty in the relay nodes.} There have been some
recent works to address the half-duplex issue in cooperative relay
systems. For example, complex echo cancelation technique is used at
the relay to cancel the coupled interference from the transmitting
path \cite{Yip:87,Vega:08}. However, these works all focused at the
physical layer signal processing. In \cite{ErnestLo:07}, the authors
exploit special topology and proposed some relay protocols to get
rid of the half-duplex penalty. Moreover, this approach depends
heavily on the locations of the relays and it cannot be extended to
general relay channel. In this paper, we are interested to explore a
system level solution to deal with the half-duplex issue. We
consider a simple MIMO cooperative relay system with a multi-antenna
source node (Src), $M$ multi-antenna relay nodes (RS) and a
multi-antenna destination node (Dst). We shall illustrate that relay
buffering  can be utilized to significantly reduce the intrinsic
half-duplex penalty. Since buffering is involved, it is important to
consider not only the throughput performance but also the associated
end-to-end delay performance. As a result, we shall focus on
delay-optimal resource control for the two-hop protocol in MIMO
cooperative relay systems.

Delay-optimal resource control in cooperative relay system is a very
difficult problem. Most of the existing works have assumed infinite
backlogs of information and focus on optimizing the throughput
performance only.
A systematic approach is to model the delay-optimal control as
Markov Decision Process (MDP) \cite{Bertsekas:1987,Cao:2008}.
\textcolor{black}{However}, there is a well-known issue of the {\em
curse of dimensionality} and brute force value iteration or policy
iteration could not give simple implementable solutions\footnote{For
example, for a system with maximum buffer length of $20$, $3$ CSI
states and $M$ RSs, the total number of system states is
$20^{M+1}\times 3^{2M}$, which is unmanageable even for small number
of RS.}. For multi-hop systems, there is a unique challenge
concerning the complex interactions of buffers at the source node
and the $M$ RS nodes and the existing solutions for single-hop
systems cannot be extended easily to deal with this situation. There
are a few recent works that considered queue dynamics in relay
systems \cite{Yeh:05,Georgiadis-Neely-Tassiulas:2006}. However,
these works have focused on the characterization of the {\em
stability region} and throughput optimal control. The question of
delay-optimal control for cooperative relay system remains to be
open. In addition, another important technical challenge is the
distributive implementation consideration. For instance, the entire
system state could be characterized by the {\em global CSI} (CSI
among every pair of nodes in the system) as well as the {\em global
QSI} (QSI of every buffer in the system).
\textcolor{black}{Brute-force solution of the MDP will yield a
control policy that is} adaptive to the global CSI and global QSI.
This poses a huge implementation challenges because these global
system state information are distributed
locally at each of the source and relay nodes. 

In this paper, we shall address the above challenges as follows. We
shall first formulate the delay-optimal resource control policy
(such as the power control and RS selection) as an average-cost
infinite horizon Markov Decision Process (MDP). To alleviate the
{\em curse of dimensionality}, and to obtain a distributive and low
complexity solution, we first introduce a {\em per-node value
function} to approximate the value function of the associated
Bellman equation. Based on the per-node value function, we derive a
distributive {\em two-stage two-winner auction-based} control
policy, which is a function of the local CSI and local QSI. The
per-node value function is obtained via a distributive online
stochastic learning algorithm, which requires local CSI and local
QSI only. The proposed online stochastic learning is quite different
from the conventional reinforced learning \cite{Abounadi:98} in
mainly two ways: (1) We are dealing with {\em constrained MDP}
(CMDP) and our online iterative solution updates both the value
function and the Lagrange multipliers (LM) simultaneously; (2) The
control action is determined from the per-node value function of all
the nodes via a per-slot auction  mechanism. Therefore, the
algorithm dynamics of the per-node online learning is not a {\em
contraction mapping} and hence, standard convergence proof using
fixed point theorem cannot be applied in our case directly. Using
the technique of separation of different time scales, we establish
technical conditions for the almost sure convergence of the proposed
distributive stochastic learning. We also show that the proposed low
complexity distributive solution is asymptotically global optimal
under heavy traffic loading. Finally, we demonstrate by simulation
that the proposed scheme has significant performance gain over
various baselines (such as conventional {\em CSIT-only} control and
the {\em throughput-optimal control} (in stability sense)) with low
complexity $\mathcal{O}(M)$ and low signaling overhead.

\section{System Models} \label{sec:model}

\subsection{System Architecture and MIMO Relay Physical Layer Model}

\begin{figure}
\centering
\includegraphics[height=4.5cm, width=11cm]{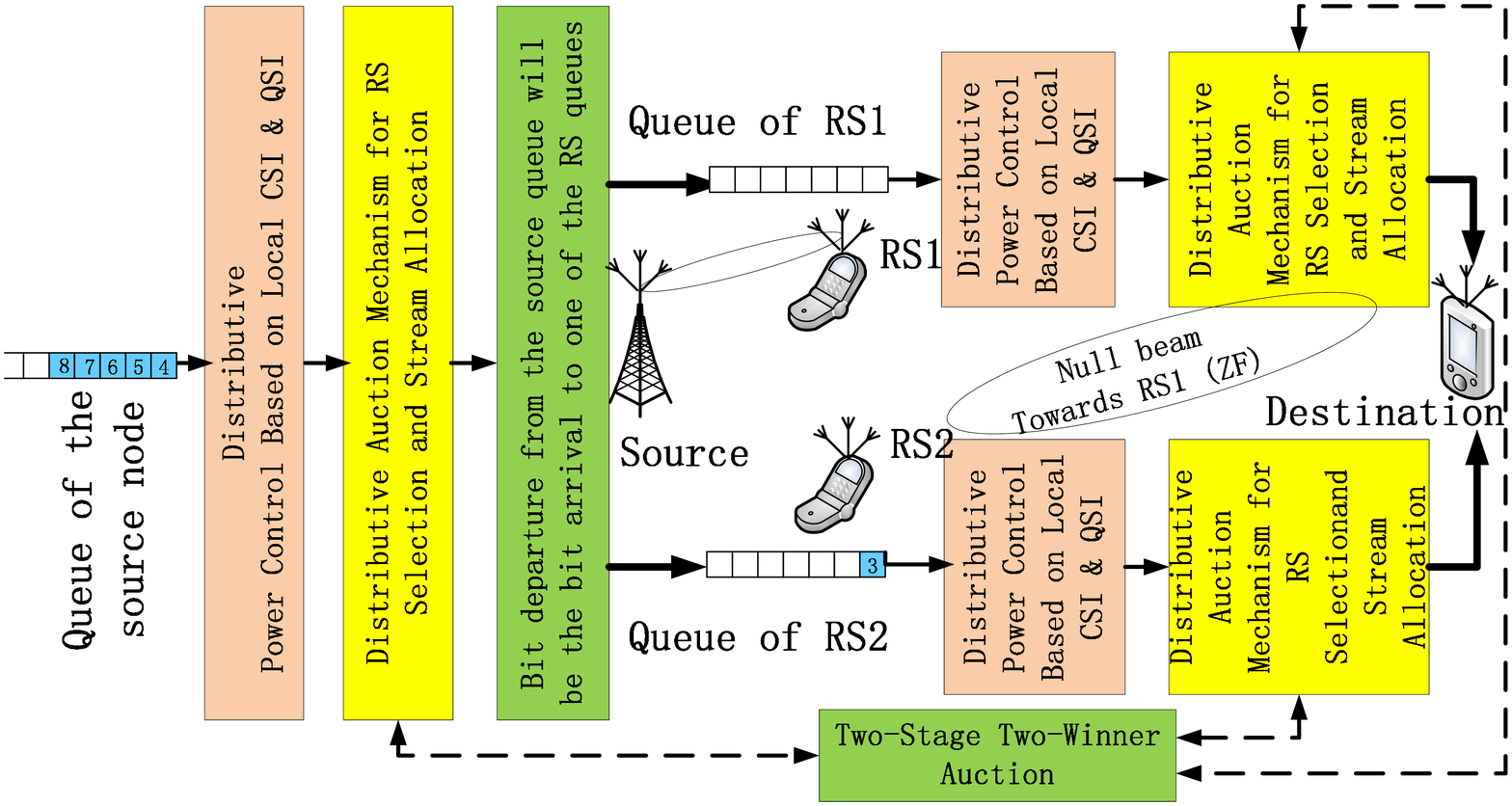}
\caption{Illustration of the two-hop MIMO cooperative system with a
multi-antenna source node, 2 multi-antenna RS nodes and a
multi-antenna destination node. By exploiting buffers at the 2 MIMO
RSs, the S-R link (source node to RS1) and R-D link (RS2 to
destination node) can deliver packets simultaneously
\textcolor{black}{without interfering with each other using signal
processing techniques (with appropriate precoder and decorrelator
designs)}. Thus,
by exploiting relay buffering, we could substantially reduce the intrinsic penalty associated with half duplex relays. 
}
\label{Fig:system}
\end{figure}

We consider a two-hop multi-antenna cooperative relay communication
system with one multi-antenna source node ($N_T$ antennas), $M$
multi-antenna half-duplex relay stations (RS, each with $N_R$
antennas) and one multi-antenna destination node ($N_T$ antennas),
as illustrated in Fig. \ref{Fig:system}. The source node cannot
deliver packets directly to the destination node due to limited
coverage and the cooperative RSs are deployed to extend the source
node's coverage.



Denote the Rx-RS and the Tx-RS as the $m$-th RS and the $n$-th RS
for notation simplicity\footnote{Since the RSs are half-duplex under
practical consideration, we require $m\neq n$ implicitly.}. Let
$N_{SR}$ and $N_{RD}$ be the number of data streams transmitted in
the S-R link and the \textcolor{black}{R-D} link respectively, where
we require $N_{RD}=\min(N_T,N_R-N_{SR})$ for simultaneous
interference-free transmission. We shall illustrate the signal model
of the S-R$_m$ link and the R$_n$-D link as follows:
\begin{itemize}
\item \textbf{S-R$_m$ link}: Let $\mathbf{X}_S \in C^{N_{SR} \times 1}$ and $\mathbf{F}_S\in
C^{N_T \times N_{SR}}$ be the symbol vector and the precoder matrix
of the source node respectively,
 $\mathbf{G}_m \in C^{N_{SR} \times N_R}$ be the decorrelator matrix at the
$m$-th RS node, the $N_{SR} \times 1$ post-processing symbol vector
at the $m$-th RS is given by $\mathbf{Y}_m = \mathbf{G}_m
\mathbf{H}_{S,m} \mathbf{F}_S \mathbf{X}_S + \mathbf{Z}_{S,m}$,
where $\mathbf{H}_{S,m} \in C^{N_R \times N_T}$ is the zero-mean
unit variance i.i.d. complex Gaussian fading matrix from the source
node to the $m$-th RS, $\mathbf{Z}_{S,m}\in C^{N_{SR} \times 1}$ is
the zero-mean unit variance complex Gaussian channel noise.

\item \textbf{R$_n$-D link}: Let $\mathbf{X}_n \in C^{N_{RD}\times 1}$ and
$\mathbf{F}_n \in C^{\textcolor{black}{N_R} \times N_{RD}}$ be the
transmit symbol vector and the precoder of the $n$-th RS
respectively, the $\textcolor{black}{N_T} \times 1$ received symbol
vector at the destination node is given by\footnote{Due to the
limited coverage of the source node, we assume the received signal
from the source node is negligible compared with the received signal
from the relay node.} $\mathbf{Y}_D =  \mathbf{H}_{n,D} \mathbf{F}_n
\mathbf{X}_n + \mathbf{Z}_{n,D}$,
where $\mathbf{H}_{n,D} \in C^{N_T \times N_R}$ is complex Gaussian
fading matrix from the $n$-th RS to the destination node,
$\mathbf{Z}_{n,D} \in C^{N_T \times 1}$ is the complex Gaussian
channel noise.
\end{itemize}

In this paper, the resource control is performed distributively on
each RS and therefore, we define the local channel state information
(CSI) available at each RS as follows. \textcolor{black}{For the
$m$-th RS, there are two types of {\em local CSI}, namely the {\em
type-I local CSI} and {\em type-II local CSI} as illustrated in Fig.
\ref{Fig:frame}. The type-I and type-II local CSI of the $m$-th RS
are denoted by $\mathbf H_m^I= \{\mathbf{H}_{S,m} \}$ and $\mathbf
H_m^{II} =\{\mathbf{H}_{m,D}\}\cup \{\mathbf H_{m,n}| n\neq m, 1\leq
n \leq M\}$, respectively. For notation convenience, let
$\mathbf{H}_m = \mathbf H_m^I\cup\mathbf H_m^{II}$ be the local
CSI\footnote{\textcolor{black}{Note that both the type-I and type-II
local CSI  at the $m$-th RS refers to all the outgoing links from
the m-th RS and hence, they can be measured at the m-th RS using
channel reciprocity and preambles. For example, there are standard
signaling and channel sounding mechanisms in the WiMAX (802.16j,
802.16m) and LTE systems for the RS to acquire the local CSI.}} at
the $m$-th RS and $\mathbf{H} =\cup_{m=1}^M \ \mathbf{H}_m$ be the
global CSI (GCSI) of the system.}
Moreover, the assumption on the channel is summarized below:
\begin{Assumption}[Assumption on Channel Fading]\textit{
We assume the channel fading elements in the global CSI $\mathbf{H}$
are i.i.d. $\mathcal{CN}(0,1)$. The CSI is quasi-static within a
frame but i.i.d. between frames.} \label{asm:channel}
\end{Assumption}

\textcolor{black}{We} assume strong channel coding is used and
hence, the maximum achievable data rate is given by the
instantaneous mutual information\footnote{For example, LDPC with
reasonably large block length (e.g 8kbyte) can achieve the
instantaneous mutual information within 0.5dB SNR
\cite{Richardson:01}.}. If the source node transmits $R_{S,m}$
information bits to the $m$-th RS in the current frame, the frame
will be successfully received if $ R_{S,m} \leq \tau \log_2 \det
\big[ \mathbf{I}+ \mathbf{G}_m \mathbf{H}_{S,m} \mathbf{F}_S
\mathbf{F}_S^{\dag} \mathbf{H}_{S,m}^{\dag} \mathbf{G}_m^{\dag}\big]
$, where $^{\dag}$ denotes the matrix conjugate transpose  and
$\tau$ is the frame duration. Similarly, the destination node could
successfully decode a frame with $R_{n,D}$ information bits
(transmitted from the $n$-th RS) if $ R_{n,D} \leq \tau \log_2 \det
\big[ \mathbf{I}+ \mathbf{H}_{n,D} \mathbf{F}_n \mathbf{F}_n^{\dag}
\mathbf{H}_{n,D}^{\dag} \big]$.

\subsection{\textcolor{black}{Buffered Decode and
Forward}}\label{subsec:BDF} \textcolor{black}{Although the RS nodes
are half-duplex relays\footnote{Half-duplex relay means that the RS
nodes do not have any Tx/Rx echo-cancelation capability.}, it is
possible to reduce the system half-duplex penalty by exploiting
buffers at the half-duplex RSs. Specifically, the source node could
transmit a packet to the $m$-th RS (denoted as the Rx-RS) and at the
same time, the $n$-th RS (denoted as the Tx-RS) transmits its {\em
buffered packet} to the destination node without interfering the
Rx-RS. This is possible by means of precoder-decorrelator designs at
the source node, Rx-RS ($m$-th RS) and the Tx-RS ($n$-th RS).} Let
$p_{S,m}$ and $p_{n,D}$ denote the total transmit power at the
source node for the S-R$_m$ link and the Tx-RS
\textcolor{black}{($n$-th RS)} for the
R$_n$-D link, respectively. 
For any given $N_{SR}$, $p_{S,m}$ for the S-R$_m$ link
\textcolor{black}{as well as} $N_{RD}$, $p_{n,D}$ for the R$_n$-D
link (where $N_{RD}=\min(N_T,N_R-N_{SR})$ implicitly), the
decorrelator and precoder designs are elaborated below.
\begin{itemize}
\item {\bf Precoder and Decorrelator Design of the S-R$_m$ Link \textcolor{black}{at the Rx-RS Node}}\footnote{\textcolor{black}{Type-I local CSI $\mathbf H^I_m$ is required at the $m$-th Rx-RS node to compute the precoder and decorrelator of the S-R$_m$ link.}}: The precoder at the source node ($\mathbf{F}_S$) and the
decorrelator at the Rx-RS node ($\mathbf{G}_m$) are chosen to
optimize the mutual information of the S-R$_m$ link subject to the
transmit power constraint as follows:
\begin{eqnarray}
\{\mathbf{G}_m^*(N_{SR}),\mathbf{F}_S^{*}(N_{SR}, p_{S,m})\} & = &
\arg \max_{\mathbf{F}_S, \mathbf{G}_m} \log_2 \det \Big[ \mathbf{I}
+ \mathbf{G}_m(N_{SR}) \mathbf{H}_{S,m} \mathbf{F}_S
\mathbf{F}_S^{\dag}
\mathbf{H}_{S,m}^{\dag} \mathbf{G}_m^{\dag}(N_{SR}) \Big] \nonumber\\
&s.t.& tr(\mathbf{F}_S \mathbf{F}_S^{\dag}) = p_{S,m} \quad
\mbox{(Transmit power constraint)} \label{eqn:s-r-power}.
\end{eqnarray}
Let $\mathbf{H}_{S,m} = \mathbf{U}_{S,m} \mathbf{\Sigma}_{S,m}
\mathbf{V}_{S,m}^{\dag}$ be the SVD decomposition of channel matrix
$\mathbf{H}_{S,m}$, where the singular values in
$\mathbf{\Sigma}_{S,m}$ are sorted in a decreasing order along the
diagonal,
$\mathbf{U}_{S,m}=[\mathbf{u}_{S,m}^1,...,\mathbf{u}_{S,m}^{N_R}]$
and
$\mathbf{V}_{S,m}=[\mathbf{v}_{S,m}^1,...,\mathbf{v}_{S,m}^{N_T}]$.
Using standard optimization techniques \cite{Boyd:04}, the source
precoder $\mathbf{F}_S^*$ is given by
\begin{equation}
\mathbf{F}_S^{*}(N_{SR}, p_{S,m}) =
[\mathbf{v}_{S,m}^1,...,\mathbf{v}_{S,m}^{N_{SR}}] \times diag
\Big\{\frac{1}{\lambda_{S,m}}-\frac{1}{\eta_{S,m}^1},...,\frac{1}{\lambda_{S,m}}-\frac{1}{\eta_{S,m}^{N_{SR}}}\Big\}
\label{eqn:wf-s},
\end{equation}
where $\eta_{S,m}^1\geq\eta_{S,m}^2 ...\geq \eta_{S,m}^{N_{SR}}$ are
the first $N_{SR}$ singular values of channel matrix
$\mathbf{H}_{S,m}$, $\lambda_{S,m}$ is the Lagrange multiplier
corresponding to the transmit power constraint in
(\ref{eqn:s-r-power}). The decorrelator $\mathbf{G}_m^*$ is given by
\begin{equation}
\mathbf{G}_m^*(N_{SR}) =
[\mathbf{u}_{S,m}^1,...,\mathbf{u}_{S,m}^{N_{SR}}]^{\dag}
\label{eqn:decorrelator}.
\end{equation}

\item {\bf Precoder Design of \textcolor{black}{the R$_n$-D Link} at the Tx-RS Node}\footnote{\textcolor{black}{Type-II local CSI $\mathbf H^{II}_n$ is required at the $n$-th Tx-RS node to compute the precoder of the R$_n$-D link.}}: Similarly, given the decorrelator $\mathbf{G}_m^*$ in (\ref{eqn:decorrelator}), the precoder at the Tx-RS node $\mathbf{F}_n \in C^{N_{RD} \times N_R}$
is selected to maximize R-D link mutual information subject to the
transmit power constraint and the interference nulling constraint
(at the Rx-RS node) as follows
:
\begin{eqnarray}
\mathbf{F}_n^{*}(N_{RD}, p_{n,m}) & = & \arg \max_{\mathbf{F}_n}
\log_2 \det \Big[ \mathbf{I} + \mathbf{H}_{n,D} \mathbf{F}_n
\mathbf{F}_n^{\dag}
\mathbf{H}_{n,D}^{\dag} \Big] \nonumber\\
&s.t.& \mathbf{G}_m^* (N_{SR}) \mathbf{H}_{n,m} \mathbf{F}_n = 0 \quad \mbox{(Interference nulling constraint)} \label{eqn:null}\\
&& tr(\mathbf{F}_n \mathbf{F}_n^{\dag}) = p_{n,D} \quad
\mbox{(Transmit power constraint)}\label{eqn:r-d-power}
\end{eqnarray}
The interference nulling constraint in (\ref{eqn:null}) is to allow
simultaneously active R-D and S-R links using half-duplex RSs. Let
$\mathbf H_{n,D} \times null(\mathbf{G}_m \mathbf{H}_{n,m}) =
\mathbf{U}_{n,D} \mathbf{\Sigma}_{n,D} \mathbf{V}_{n,D}^{\dag}$ be
the SVD decomposition, where the singular values in
$\mathbf{\Sigma}_{n,D}$ are sorted in a decreasing order along the
diagonal, $null(\mathbf{G}_m \mathbf{H}_{n,m})$ denotes the null
space of matrix $\mathbf{G}_m \mathbf{H}_{n,m}$ and
$\mathbf{V}_{n,D}=[\mathbf{v}_{n,D}^1,...,\mathbf{v}_{n,D}^{N_R-N_{SR}}]$.
Using standard optimization techniques \cite{Boyd:04}, the precoder
at the Tx-RS ($\mathbf{F}_n^*$) is given by:
\begin{equation}
\mathbf{F}_n^{*} (N_{RD}, p_{n,m}) =
[\mathbf{v}_{S,n}^1,...,\mathbf{v}_{S,n}^{N_{RD}}] \times diag
\Big\{\frac{1}{\lambda_{n,D}}-\frac{1}{\eta_{n,D}^1},...,\frac{1}{\lambda_{n,D}}-\frac{1}{\eta_{n,D}^{N_{RD}}}\Big\}
\label{eqn:wf-r},
\end{equation}
where $\eta_{n,D}^1\geq \eta_{n,D}^2...\geq \eta_{S,m}^{N_{RD}}$ are
the first $N_{RD}$ singular values of channel matrix $\mathbf
H_{n,D} \times null(\mathbf{G}_m \mathbf{H}_{n,m})$, $\lambda_{n,D}$
is the Lagrange multiplier corresponding to the power constraint in
(\ref{eqn:r-d-power}).
\end{itemize}

\subsection{Bursty Source Model and  Queue Dynamics}
There is one queue in the source node and one queue in each of the
$M$ RSs respectively for the storage of received information bits.
Let $N_Q$ be the maximum buffer size (number of bits) for the
buffers in the source node and all the RSs. Let $X(t)$ indicates the
\textcolor{black}{number of} new information bits arrival in the
$t$-th frame at the source node. The assumption on the bit arrival
process is given below:
\begin{Assumption}[Assumption on Arrival Process]\textit{
We assume $X(t)$ is i.i.d. over frames based on a general
distribution $f_X(x)$ \textcolor{black}{with $\mathbb E[X(t)] =
\lambda_S$} and the information bits \textcolor{black}{arrive} at
the end of each frame.} \label{asm:arrival}
\end{Assumption}
Moreover, let $Q_S(t)$ and $Q_m(t)$ denote the number of information
bits in the source node's queue and  the $m$-th RS's  queue ($1\leq
m \leq M$) at frame $t$. We assume each RS has the knowledge of its
own queue length and the source node's queue length. Thus, the local
QSI of the $m$-th RS is $\big(Q_S(t),Q_m(t)\big)$.
$\mathbf{Q}(t)=\big(Q_S(t), Q_1(t), \cdots, Q_M(t)\big)$ denotes the
\textit{global queue state information} (GQSI) at frame $t$.

The overall system queue dynamics at the source node and the RSs are
summarized below.
\begin{itemize}
\item If the source node successfully delivers $R_{S,m}(t)$ information bits to the $m$-th
RS at frame $t$, then $Q_S(t+1) = \min \left\{\max\{Q_S(t)-
R_{S,m}(t),0\} + X(t), N_Q \right\}$ and $Q_m(t+1) = \min
\left\{Q_m(t) + R_{S,m}(t), N_Q \right\}$.

\item If the source node fails to deliver any information bit to the RSs , then $
Q_S(t+1) = \min \left\{Q_S(t) + X(t), N_Q \right\}$.

\item If the $n$-th RS successfully delivers $R_{n,D}(t)$ information bits to the
destination at frame $t$, then $Q_n(t+1) = \max\{Q_n(t) -
R_{n,D}(t), 0\}$.
\end{itemize}

\begin{Remark}
Each information bit delivered from the source node will be received
by one of the RSs and different RSs may have different information
bits in the buffer. When the source node is to deliver information
bits to one RS, selecting different RSs with different buffer
lengths may have different effects on the average packet delay of
the system. Therefore, not only the CSI of all S-R links but also
the QSI of all RSs should be considered in directing the source
node's  transmission. Such coupling on the system QSI is a unique
challenge in delay-optimal control of multi-hop systems. Fig.
\ref{Fig:system} shows the top level architecture illustrating the
interactions among all the queues in the two-hop cooperative system.
\end{Remark}

\subsection{Distributive Contention Protocol}\label{subsec:model-prot}


\textcolor{black}{Based on the BDF in Section \ref{subsec:BDF}, we
still need to determine (a) which RS should be the Rx-RS ($m^*$),
(b) which RS should be the Tx-RS ($n^*$) and (c) the number of data
streams transmitted by the source node ($N_{SR}^*$) and the Tx-RS
($N_{RD}^*$). Due to the decentralized control requirement, we shall
propose a {\em two-stage two-winner auction} mechanism for
distributive\footnote{\textcolor{black}{Similar to the common notion
of distributive algorithms in the literature
\cite{HanZhu:07,Palomar:07}, the term ``distributive'' in this paper
refers to algorithms that perform computation locally but require
explicit message passing. Yet, the message passing overhead in the
bidding process is quite
mild\cite{Huangjianwei:06,CurescuBidding:08}.}} contention
resolution.}

\begin{figure}
\centering
\includegraphics[height=8cm, width=11cm]{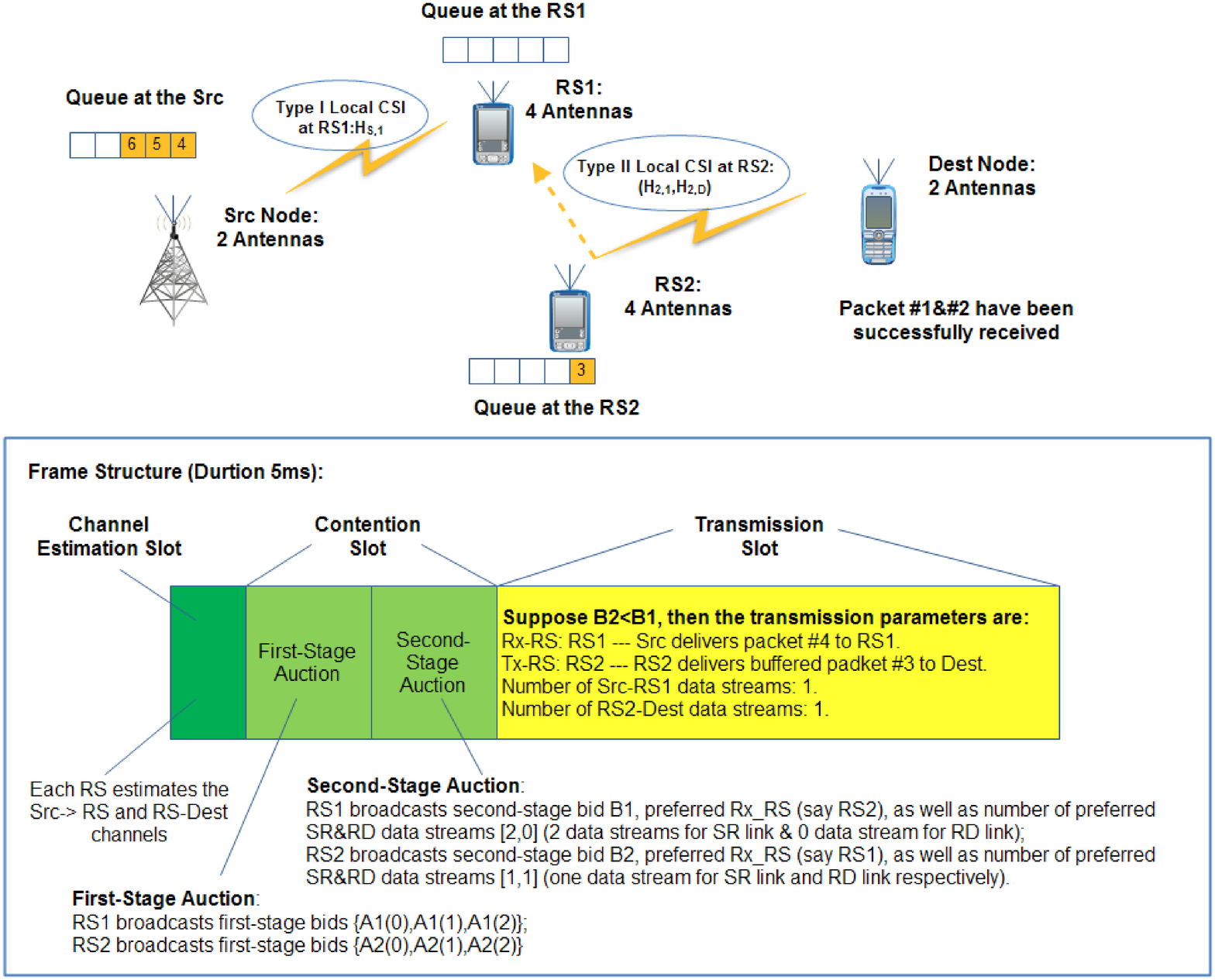}
\caption{\textcolor{black}{Illustration of an example of bidding
protocol for a 2-RS system. } } \label{Fig:frame}
\end{figure}

\textcolor{black}{Figure \ref{Fig:frame} illustrates an example of
bidding protocol for a 2-RS system.} As a result, the RS selection
and data stream allocation procedure can be parameterized by a
bidding vector
$\big\{\big(\textcolor{black}{A_m(0),...,A_m(\min(N_T,N_R))}\big),$
$B_{m})|\forall m\big\}$. We shall refer the bidding vector as the
{\em RS selection and data stream allocation policy} in the rest of
the paper.


%

\subsection{Optimization Objective and Control Policy}

\begin{Definition}[Distributive Stationary Control Policy]
\textit{A distributive stationary control policy $\Pi=\{\Pi^m|1 \leq
m \leq M\}$ \textcolor{black}{is a collection of stationary control
policies $\Pi^m$ at the $m$-th RS, where}
$\Pi^m=\{\Pi^m_p,\Pi^m_A,\Pi^m_{B}\}$ includes
\textcolor{black}{the} power allocation policy of S-R link and R-D
link $\Pi_p^m$, the first-stage and second-stage bidding policy
$\Pi_A^m$ and $\Pi_B^m$. Specifically,
\begin{align}
&\Pi^m_p(S_m)=\Big\{p_{S,m}(N_{SR}),p_{m,D}(N_{RD}):N_{SR},N_{RD}=0,1,\cdots,\min(N_T,N_R)\Big\}\triangleq\mathbf p_m\label{eqn:policy-p}\\
&\Pi^m_A(S_m)=\Big\{A_m(N_{SR})|N_{SR}=0,1,...,\min(N_T,N_R)\Big\}\triangleq\mathbf
A_m\label{eqn:policy-A}\\
&\Pi^m_B (S_m,\cup_{m'=1}^M \mathbf A_{m'})= \Big\{B_m,
I_m, (N_{SR,m},N_{RD,m})|N_{RD}=\min(N_T,N_R-N_{SR})\Big\}\triangleq\mathbf B_m\label{eqn:policy-B}
\end{align}
for $m=1,2,...,M$, where $p_{S,m}(N_{SR})$ is the total transmit
power allocation at the source node for the S-R link with $N_{SR}$
data streams, $p_{m,D}(N_{RD})$ is the total transmit power
allocation at the Tx-RS for the R-D link with $N_{RD}$ data
streams.
} \label{def:policy}
\end{Definition}

\textcolor{black}{Denote the local system \textcolor{black}{state}
of the $m$-th RS as $S_m=(Q_S,Q_m,\mathbf{H}_m)$ ($1\leq m \leq
M$). Therefore, the global system state is given by $\mathbf
S=\cup_{m=1}^M \ S_m= (\mathbf{Q},\mathbf{H})$.}

\begin{Remark} [Distributive Consideration of Stationary Control Policy $\Pi$ in Definition
\ref{def:policy}] The stationary control policy $\Pi=\{\Pi^m|1 \leq
m \leq M\}$ is distributive 
in the sense that the policy $\Pi^m$ at each RS $m$ only depends on
the local system state $\textcolor{black}{S_m}$ and the broadcast
bidding information available at RS $m$. Thus, for notation
simplicity, we shall omit the bidding \textcolor{black}{information}
when the meaning is clear, i.e. we shall use
$\Pi^m(\textcolor{black}{S_m})=\{\Pi^m_p(\textcolor{black}{S_m}),\Pi^m_A(\textcolor{black}{S_m}),\Pi^m_{B}(\textcolor{black}{S_m})\}$
in the rest of
the paper. 
\end{Remark}

A stationary control policy $\Pi$ induces a joint distribution for
the random process $\{\mathbf S(t)\}$. Under Assumption
\ref{asm:channel} and \ref{asm:arrival}, $\mathbf S(t+1)$ only
depends on $\mathbf S(t)$ and actions at frame $t$, and hence the
induced random process $\{\mathbf S(t)\}$ for a given control policy
$\Pi$ is Markovian with the following transition probability:
\begin{align}
\Pr \big[\mathbf S(t+1)\big| \mathbf S(t), \Pi\big(\mathbf S(t)\big)
\big]=\Pr\big[\mathbf{H}(t+1)\big]\Pr \big[\mathbf{Q}(t+1)\big|
\mathbf S(t), \Pi\big(\mathbf S(t)\big)\big]\label{eqn:kernel},
\end{align}
where the equality is because of Assumption \ref{asm:channel} and
the queue dynamics transition probability $\Pr
\big[\mathbf{Q}(t+1)\big| \mathbf S(t), \Pi\big(\mathbf
S(t)\big)\big]$ is given by
\begin{align}
&\Pr \big[\mathbf{Q}(t+1)\big| \mathbf S(t),
\Pi\big(\mathbf S(t)\big)\big] \label{eqn:queue-tran}\\
=&\begin{cases} \Pr\big[X(t)=
Q_S(t+1)-[Q_S(t)-\textcolor{black}{R_{S,m^*}}(t)]^+\big],&\text{if }
Q_m(t+1)=Q_m(t)\ (\forall m\neq m^*,n^*)\\
\quad \quad  \text{and } Q_{m^*}(t+1) = \min \{Q_{m^*}(t) +
R_{S,m^*}(t), N_Q \}, &Q_{n^*}(t+1) = \max\{Q_{n^*}(t) -
R_{n^*,D}(t), 0\}\\
0, &\text{otherwise}
\end{cases}\nonumber
\end{align}
Given a unichain policy $\Pi$, the induced Markov chain $\{\mathbf
S(t)\}$ is ergodic and there exists a unique steady state
distribution $\pi_{S}$\cite{Bertsekas:1987} . Therefore, we have the
average end-to-end delay  of the two-hop cooperative RS system
summarized in the following lemma:
\begin{Lemma} [Average End-to-End Delay]
For small average packet drop rate constraint $D$, the average
end-to-end delay  of the two-hop cooperative RS system is given by
\begin{equation}
\overline{T}(\Pi) = \lim\limits_{T \rightarrow \infty }
\frac{1}{T}\sum_{t=1}^T \mathbf{E}^{\Pi}\Big[\frac{\sum_{m=S}^M
Q_m(t)}{\lambda_S}\Big]=\mathbf{E}_{\pi_{S}}\Big[\frac{\sum_{m=S}^M
Q_m}{\lambda_S}\Big] \label{eqn:avr-delay}
\end{equation}
where $m=S,1,2,...,M$ in the equation\footnote{This abuse will also
appear in the following of this paper as long as the meaning is
clear.}, $\mathbf{E}_{\pi_{S}}$ means taking the expectation with
respect to the induced steady state distribution $\pi_{S}$ (induced
by the unichain control policy $\Pi$) and $\lambda_S$ is the average
number of arrival bits per frame at the source node.\label{lem:obj}
\end{Lemma}
\begin{proof}
\textit{Please refer to Appendix A. }
\end{proof}
Similarly, the source node's average drop rate
constraint\footnote{Since the source node and $M$ RSs have buffers
with the same buffer size $N_Q$, the average drop rate at each RS
node is much lower than the average drop rate at the source node.
Therefore, we omit the average drop rate constraint at each RS to
simplify the problem. }, the source node's average power constraint
and each RS $m$'s average power constraint are given by
\begin{align}
& \overline{D}(\Pi)=\lim\limits_{T \rightarrow \infty }
\frac{1}{T}\sum_{t=1}^T
\mathbf{E}^{\Pi}\Big[\mathbf{I}[Q_S(t)=N_Q]\Big]=
\mathbf{E}^{\Pi}_{\pi_{\mathbf S}}\Big[\mathbf{I}[Q_S=N_Q]\Big] \leq
D \label{eqn:d-cont}\\
& \overline{P}_S(\Pi) = \lim\limits_{T \rightarrow \infty }
\frac{1}{T}\sum_{t=1}^T \mathbf{E}^{\Pi}\Big[\sum_{m=1}^{M}
\sum_{i=1}^{N_{\min}} I_{S,m}^i(t) p_{S,m}(i)(t)\Big] =
\mathbf{E}_{\pi_{\mathbf S}}^{\Pi} \Big[\sum_{m=1}^{M}
\sum_{i=1}^{N_{\min}} I_{S,m}^i p_{S,m}(i)\Big] \leq
P_S\label{eqn:s-p-cont}\\
 & \overline{P}_m(\Pi)=\lim\limits_{T \rightarrow \infty }
\frac{1}{T}\sum_{t=1}^T \mathbf{E}^{\Pi} \Big[\sum_{i=1}^{N_{\min}}
I_{m,D}^i(t) p_{m,D}^i(t)\Big]= \mathbf{E}_{\pi_{\mathbf S}}
\Big[\sum_{i=1}^{N_{\min}} I_{m,D}^i p_{m,D}^i\Big]  \leq P_R, 1\leq
m \leq M \label{eqn:r-p-cont}
\end{align}
where $N_{\min}=\min(N_T,N_R)$, $I_{S,m}^i=\mathbf I [m=m^*]\mathbf
I[i=N_{SR}^*]$ and $I_{m,D}^i=\mathbf I [m=n^*]\mathbf
I[i=N_{RD}^*]$.

\section{Constrained Markov Decision Problem Formulation} \label{sec:problem}
In this section, we shall formulate the delay-optimal problem as an
infinite horizon average cost constrained Markov Decision Problem
(CMDP) and discuss the general solution.

\subsection{CMDP Formulation}

The goal of the controller is to choose an optimal stationary
feasible unichain policy $\Pi^*$ that minimizes the average
end-to-end transmission delay in \eqref{eqn:avr-delay}.
Specifically, the delay-optimal control problem is summarized below.
\begin{Problem}[Delay-Optimal Control Problem for MIMO Relay System] Find a feasible stationary unichain policy $\mathbf{\Pi}=(\Pi^1, ...,\Pi^{M})$
such that the average end-to-end delay is minimized subject to the
average drop rate constraint at the source node and the average
power constraint at the source node and each RS node\footnote{To
simplify the notation, we shall normalize $\lambda_S=1$ in the rest
of the paper.}, i.e. $\min\limits_{\Pi} \ \overline{T}(\Pi) =
\lim\limits_{T \rightarrow \infty } \frac{1}{T}\sum_{t=1}^T
\mathbf{E}^{\Pi}\Big[\sum_{m=S}^M Q_m(t)\Big]\ s.t. \
\eqref{eqn:d-cont},\ \eqref{eqn:s-p-cont}, \ \eqref{eqn:r-p-cont}$.
\label{prob:main}
\end{Problem}

Problem \ref{prob:main} is an infinite horizon average cost
constrained Markov Decision Problem (CMDP) \cite{Bertsekas:07} with
system state space $\mathcal S=\{\mathbf S^1,\mathbf
S^2,\cdots\}=\mathcal Q\times \mathcal H$ (where $\mathcal Q$ is the
global QSI state space and $\mathcal H$ is the global CSI state
space), action space $\mathcal P\times\mathcal A \times \mathcal B $
(where $\mathcal P=\{\forall \mathbf p_m| \forall m\}$ is power
allocation action space, $\mathcal A=\{\forall \mathbf A_m| \forall
m\}$ is the first-stage bidding action space and $\mathcal
B=\{\forall \mathbf B_m| \forall m\}$ is the second-stage bidding
action space), transition kernel given by \eqref{eqn:kernel}, and
the per-stage cost function $d\big(\mathbf S, \Pi (\mathbf
S)\big)=\sum_{m=S}^M Q_m$.

\subsection{Lagrangian Approach for the CMDP}
The CMDP in Problem \ref{prob:main} can be converted into
unconstrained MDP by the Lagrange theory \cite{Boyd:04}. For any
vector of Lagrange multiplier (LM)
$\gamma=[\gamma_{S,d},\gamma_{S,p},\gamma_{1,p},...,\gamma_{M,p}]^T$,
we define the Lagrangian as  $L(\Pi,\gamma) = \lim\limits_{T
\rightarrow +\infty } \frac{1}{T} \sum_{t=1}^T
\mathbf{E}^{\Pi}_{\mathbf S}\Big[g\Big(\mathbf S(t),\Pi\big(\mathbf
S(t)\big),\gamma\Big)\Big]$, where
\begin{align}
g\big(\mathbf S,\Pi\big(\mathbf S\big),\gamma \big) = Q_S  +
\gamma_{S,p} \sum_{m=1}^M \sum_{i=1}^{N_{\min}} I_{S,m}^i p_{S,m}(i)
+ \gamma_{S,d} \mathbf{I}[Q_S=N_Q] +\sum_{m=1}^M\Big[ Q_m +
\gamma_{m,p} \sum_{i=1}^{N_{\min}} I_{m,D}^i p_{m,D}(i) \Big]
\nonumber.
\end{align}
Therefore, the corresponding unconstrained MDP for a particular
vector of LMs $\gamma$ is given by
\begin{eqnarray}
G(\gamma)=\min_{\Pi} L(\Pi, \gamma)= \min_{\Pi} \lim\limits_{T
\rightarrow \infty } \frac{1}{T} \sum_{t=1}^T
\mathbf{E}^{\Pi}\Big[g\Big(\mathbf S(t),\Pi\big(\mathbf
S(t)\big),\gamma\Big)\Big] \label{eqn:lagrange}
\end{eqnarray}
where $G(\gamma)$ gives the Lagrange dual function. The dual problem
of the primal problem in Problem \ref{prob:main} is given by
$\max_{\gamma\succeq0}G(\gamma)$. It is shown in
\cite{Borkaractorcritic:2005} that there exists a Lagrange
multiplier $\boldsymbol{\gamma}\ge 0$ such that $\Pi^*$ minimizes
$L(\Pi,\boldsymbol{\gamma})$ and the saddle point condition the
saddle point condition $ L(\Pi, \gamma^*) \geq L(\Pi^*, \gamma^*)
\geq L(\Pi^*, \gamma)$ holds. Using standard Lagrange theory
\cite{Boyd:04}, $\Pi^*$ is the primal optimal (i.e. solving Problem
\ref{prob:main}), $\gamma^*$ is the dual optimal (solving the dual
problem) and the duality gap is zero. Thus, by solving the dual
problem, we can obtain the primal optimal $\Pi^*$. Therefore, we
shall first solve the unconstrained MDP in \eqref{eqn:lagrange} in
the following.

For a given LM vector $\gamma$, the optimizing unichain policy for
the unconstrained MDP (\ref{eqn:lagrange}) can be obtained by
solving the associated {\em Bellman equation} w.r.t.
($\theta,\{J(\mathbf S)\}$) as follows
\begin{equation}
\theta + J(\mathbf S^i) = \min_{\Pi(\mathbf S^i)}
\Big\{g\big(\mathbf S^i,\Pi(\mathbf S^i),\gamma\big) + \sum_{\mathbf
S_j} \Pr[\mathbf S^j|\mathbf S^i,\Pi(\mathbf S^i)] J(\mathbf
S^j)\Big\} \quad \forall \mathbf S^i \in \mathcal S \label{eqn:MDP},
\end{equation}
where $\{J(\mathbf S)\}$ is the value function of the MDP and
$\Pr[\mathbf S^j|\mathbf S^i,\Pi(\mathbf S^i)]$ is the transition
kernel which can be obtained from \eqref{eqn:kernel},
$\theta=\min_{\Pi} L(\Pi, \gamma)$ is the optimal average cost per
stage and the optimizing policy is $\Pi^*$ with $\Pi^*(\mathbf S^i)$
minimizing the R.H.S. of \eqref{eqn:MDP} at any state $\mathbf S^i$.
For any unichain policy with irreducible Markov Chain $\{\mathbf
S(t)\}$, the solution to (\ref{eqn:MDP}) is unique
\cite{Bertsekas:07}. We restrict our policy space to be {\em
unichain policies}\footnote{For most of the policies we are
interested, the associated Markov chain is irreducible and hence,
there is virtually no loss by restricting ourselves to unichain
policies.} and we denote $\Pi^*$ as the optimal unichain policy.

\subsection{Equivalent Bellman Equation for the CMDP}


The Bellman equation in \eqref{eqn:MDP} is a fixed point problem
over the functional space and this is very complicated to solve due
to the huge cardinality of the system state space. Brute-force
solution could not lead to any useful implementations. In this
subsection, we shall illustrate that the Bellman equation in
\eqref{eqn:MDP} can be simplified into a equivalent form by
exploiting the i.i.d. structure of the CSI process $\mathbf{H}(t)$.
For notation convenience, we partition the unichain policy $\Pi$
into a collection of actions based on the QSI. Specifically, we have
the following definition.

\begin{Definition}[Partitioned Actions for the $m$-th Relay]\textit{
Given a unichain control policy $\Pi^m$, we define $\Pi^m(\mathbf
Q)=\Pi^m(Q_S,Q_m) = \left\{\Pi^m(Q_S,Q_m,\mathbf{H}_m)|\forall
\mathbf{H}_{m} \right\}$
as the collection of actions under a given local QSI $(Q_S,Q_m)$ for
all possible local CSI $ \mathbf{H}_{m}$. The complete policy
$\Pi^m$ for the $m$-th RS is therefore equal to the union of all the
partitioned actions, i.e.
$\Pi^m=\cup_{(Q_S,Q_m)}\Pi^m(Q_S,Q_m)$.}\label{Def:partintioned-action}
\end{Definition}
Therefore, we have $\Pi=\cup_{\mathbf Q}\Pi(\mathbf Q)$ and we show
that the  optimal policy $\Pi^*$  of \eqref{eqn:lagrange} can be
obtained by solving an {\em equivalent Bellman equation} summarized
in the following lemma.

\begin{Lemma}[Equivalent Bellman Equation]
The control policy obtained by solving the Bellman equation in
\eqref{eqn:MDP} is the same as that obtained by solving the {\em
equivalent Bellman equation} defined below:
\begin{equation}
\theta + V(\mathbf{Q}^i) = \min_{\Pi(\mathbf{Q}^i)}
\Big\{\overline{g}\big(\mathbf{Q}^i,\Pi(\mathbf{Q}^i),\gamma\big) +
\sum_{\mathbf{Q}_j} \Pr[\mathbf{Q}^j|\mathbf{Q}^i,\Pi(\mathbf{Q}^i)]
V(\mathbf{Q}^j)\Big\}, \forall \mathbf Q^i \in \mathcal
Q\label{eqn:bellman}
\end{equation}
where $\theta=\min_{\Pi} L(\Pi, \gamma)$ is the original optimal
average cost per stage, $V(\mathbf{Q}^i) = \mathbf{E}_{\mathbf{H}}
[J(\mathbf{Q}^i,\mathbf{H}) | \mathbf{Q}^i]$ is the conditional
average value function for state $\mathbf{Q}^i$, and
\begin{equation}
\overline{g} \big(\mathbf{Q}^i,\Pi(\mathbf{Q}^i),\gamma\big) =
\mathbf{E}_{\mathbf{H}} \Big[ g \big(
(\mathbf{Q}^i,\mathbf{H}),\Pi(\mathbf{Q}^i,\mathbf{H}),\gamma \big)
\big| \mathbf{Q}^i \Big] \label{eqn:bar-g}
\end{equation}
is the conditional per-stage cost and
$\Pr[\mathbf{Q}^j|\mathbf{Q}^i,\Pi(\mathbf{Q}^i)]=\mathbf{E}_{\mathbf{H}}
\big[\Pr[\mathbf{Q}^j|(\mathbf{Q}^i,\mathbf{H}),\Pi(\mathbf{Q}^i,\mathbf{H})]\big]$
is the conditional average transition kernel. 
\label{lem:reduce}
\end{Lemma}

\begin{proof}
\textit{Please refer to Appendix B.}
\end{proof}


\begin{Remark}
Note that  solving the R.H.S. of \eqref{eqn:bellman} for each
$\mathbf Q^i$ will get an overall control policy which is a function
of both the CSI $\mathbf H$ and QSI $\mathbf Q^i$. This is
illustrated by the following example.
\end{Remark}

\begin{Example}
Consider a simple example with global CSI state space $\mathcal
H=\{\mathbf H^1, \mathbf H^2\}$ and global QSI state space $\mathcal
Q=\{\mathbf Q^1, \mathbf Q^2\}$. Hence, the control variables are
collectively denoted by the policy $\Pi = \big\{\Pi(\mathbf
H^1,\mathbf Q^1),  \Pi(\mathbf H^2,\mathbf Q^1),$ $ \Pi(\mathbf
H^1,\mathbf Q^2), \Pi(\mathbf H^2,\mathbf Q^2)\big\}$. Using
definition \ref{Def:partintioned-action}, the partitioned actions
are simply  regroups of variables given by $\Pi(\mathbf
Q^1)=\big\{\Pi(\mathbf Q^1, \mathbf H^1),$ $\Pi(\mathbf Q^1,\mathbf
H^2)\big \}$ and $\Pi(\mathbf Q^2) = \big\{\Pi(\mathbf Q^2,\mathbf
H^1), \Pi(\mathbf Q^2,\mathbf H^2)\big\}$.
For any QSI state $\mathbf Q^i$ ($i=1,2$), using Lemma
\ref{lem:reduce}, the optimal partitioned actions $ \Pi^*(\mathbf
Q^i)$ can be obtained by solving the R.H.S. of \eqref{eqn:bellman}
as follows
\begin{align}
 \Pi^*(\mathbf Q^i)=\arg \min_{\{\Pi(\mathbf Q^i, \mathbf
H^1), \Pi(\mathbf Q^i, \mathbf H^2)\}}\Big\{&
\sum_{k=1}^2\Pr[\mathbf H^k]\Big[ g \big((\mathbf Q^i,\mathbf
H^k),\Pi(\mathbf Q^i,\mathbf
H^k),\gamma \big) \nonumber\\
&+ \sum_{\mathbf{Q}^j} \Pr \big[ \mathbf{Q}^j | (\mathbf Q^i,\mathbf
H^k), \Pi(\mathbf Q^i,\mathbf H^k) \big] V (\mathbf{Q}^j)\Big]\Big\}
\label{eqn:example-rhs}
\end{align}
Observe that the R.H.S. of \eqref{eqn:example-rhs} is a decoupled
objective function w.r.t. the variables $\{\Pi(\mathbf Q^i, \mathbf
H^1), \Pi(\mathbf Q^i, \mathbf H^2)\}$. Hence, applying standard
decomposition theory, $\forall k=1,2$, we have
\begin{align}
\Pi^*(\mathbf Q^i, \mathbf H^k)=\arg \min_{\Pi(\mathbf Q^i, \mathbf
H^k)} \Big \{g \big((\mathbf Q^i,\mathbf H^k),\Pi(\mathbf
Q^i,\mathbf H^k),\gamma\big) + \sum_{\mathbf{Q}^j} \Pr \big[
\mathbf{Q}^j | (\mathbf Q^i,\mathbf H^k), \Pi(\mathbf Q^i,\mathbf
H^k) \big] V (\mathbf{Q}^j)\Big\}\nonumber
\end{align}
Using the results in Lemma \ref{lem:reduce}, the optimal control of
the original problem when the QSI and CSI realizations are $(\mathbf
Q^1, \mathbf H^2)$ is $\Pi^*(\mathbf Q^1, \mathbf H^2)$. Hence, the
solution obtained by solving \eqref{eqn:bellman} is adaptive to both
the CSI and QSI.
\end{Example}

\section{Distributive Online Algorithm Based on Approximated MDP} \label{sec:solution}

There are still two major obstacles ahead. Firstly, obtaining the
value functions $\{V(\mathbf{Q})\}$ w.r.t. (\ref{eqn:bellman})
involves solving a system of exponential number of equations and
unknowns and brute force solution has exponential complexity.
Secondly, even if we could obtain the solution $\{V(\mathbf{Q})\}$,
the derived control actions will depend on global QSI and CSI, which
is highly undesirable. In this section, we shall overcome the above
challenges using approximate MDP and distributive stochastic
learning. The linear approximation architecture of the value
function is given below \cite{Tsitsiklis:96}:
\begin{equation}
V(\mathbf{Q}) = \sum_{m=S}^M \sum_{q=0}^{N_Q} \widetilde{V}_m(q)
\mathbf{I}[Q_m=q] \quad \mbox{or in the vector form} \quad
\mathbf{V} = \mathbf{M} \mathbf{W} \label{eqn:per-node},
\end{equation}
where we shall refer $\{\widetilde{V}_m(q)\}$  as \textit{per-node
value functions}\footnote{In this paper, we assume each RS (say the
$m$-th RS) has the knowledge of the source node's queue length $Q_S$
and its own queue length $Q_m$. Therefore, the per-node value
function $\widetilde{V}_S$ and $\widetilde{V}_m$ is known at the
$m$-th RS.} ($\forall m=S,1,\cdots,M$) and $\{V(\mathbf{Q})\}$ as
\textit{global value function} in the rest of this paper,
$\mathbf{V} = [V(\mathbf{Q}^1),...,V(\mathbf{Q}^{|\mathcal{Q}|})]^T$
is the vector form of global value functions,  the \textit{parameter
vector} $\mathbf{W}$ and \textit{mapping matrix} $\mathbf{M}$ is
given below:
\begin{eqnarray}
\mathbf{W} &=& \Big[
\widetilde{V}_S(0),...,\widetilde{V}_S(N_Q),\widetilde{V}_1(0),...,\widetilde{V}_1(N_Q),...,\widetilde{V}_S(N_Q),\widetilde{V}_M(0),...,\widetilde{V}_M(N_Q)
\Big]^T \nonumber\\
\mathbf{M} &=& \left[ \begin{array}{ccccccc}
    \mathbf{I}[\mathbf{Q}^1_S=0] & ... & \mathbf{I}[\mathbf{Q}^1_S=N_Q] & ... & \mathbf{I}[\mathbf{Q}^1_M=0] & ... & \mathbf{I}[\mathbf{Q}^1_M=N_Q]\\
    ... & ... & ... & ... & ... & ... & ... \\
    \mathbf{I}[\mathbf{Q}^{|\mathcal{Q}|}_S=0] & ... & \mathbf{I}[\mathbf{Q}^{|\mathcal{Q}|}_S=N_Q] & ... & \mathbf{I}[\mathbf{Q}^{|\mathcal{Q}|}_M=0] & ... &
    \mathbf{I}[\mathbf{Q}^{|\mathcal{Q}|}_M=N_Q]\nonumber
    \end{array} \right],
\end{eqnarray}
where we let
$\widetilde{V}_S(0)=\widetilde{V}_1(0)=...=\widetilde{V}_M(0)=0$ and
set $\mathbf Q^I=(0,\cdots,0)$ (i.e. all buffer empty) as the
reference state without loss of generality. Compared with the
original value function in (\ref{eqn:bellman}), the dimension of the
per-node value functions is much smaller. Therefore, the per-node
value function can only satisfy the Bellman equation
(\ref{eqn:bellman}) in some pre-determined system queue states. In
this paper, we shall refer the pre-determined subset of system queue
states as the \textit{representative states} \cite{Tsitsiklis:96}.
Without loss of generality, we define the reference states
$\mathcal{Q}_R = \{\beta_{m,q} | \forall m=S,1,2,...,M;
q=1,2,...,N_Q\}$, where $\beta_{m,q}$ denotes the QSI with $Q_m=q$
and $Q_n=0$ $\forall n\neq m$. Moreover, we also define the inverse
mapping matrix $\mathbf{M}^{-1}$ as
\begin{equation}
\mathbf{M}^{-1} = \left[ \begin{array}{ccccccccc}
    0 & \mathbf{I}[\mathbf{Q}^1=\beta_{S,1}] & ... & \mathbf{I}[\mathbf{Q}^1=\beta_{S,N_Q}], & ... &, 0 &\mathbf{I}[\mathbf{Q}^1=\beta_{M,1}] & ... & \mathbf{I}[\mathbf{Q}^1=\beta_{M,N_Q}]\\
    ... & ... & ... & ... & ... & ... & ... & ... & ...\\
    0 & \mathbf{I}[\mathbf{Q}^{|\mathcal{Q}|}=\beta_{S,1}] & ... & \mathbf{I}[\mathbf{Q}^{|\mathcal{Q}|}=\beta_{S,N_Q}], & ... &, 0 &\mathbf{I}[\mathbf{Q}^{|\mathcal{Q}|}=\beta_{M,1}] & ... & \mathbf{I}[\mathbf{Q}^{|\mathcal{Q}|}=\beta_{M,N_Q}]
    \nonumber
    \end{array} \right]^T.
\end{equation}
Thus, we have $\mathbf{W} = \mathbf{M}^{-1} \mathbf{V}$. Instead of
offline computing the {\em best fit} parameter vector $\mathbf{W}$
(per-node value function vector) w.r.t. the global value function
$\mathbf{V}$ (which is quite complex), we shall propose an online
learning algorithm to estimate the parameter vector $\mathbf{W}$
(per-node value function) in Section \ref{subsec:learning}.

\subsection{Distributive Control Policy under Linear Value Function Approximation}

Using the approximate value function  in (\ref{eqn:per-node}), we
shall derive a distributive control policy which depends  on the
local CSI and local QSI as well as the per-node value functions
$\{\widetilde{V}_m(q)\}$ at each node $m$ ($\forall
m=S,1,\cdots,M$). Specifically, using the approximation in
(\ref{eqn:per-node}), the control policy in (\ref{eqn:bellman}) can
be obtained by solving the following simplified optimization
problem.

\begin{Problem}[Optimal Control Action with Approximated Value Function]
For any given value function $V(\mathbf{Q}^i) = \sum_{m=S}^M
\sum_{q=0}^{N_Q} \widetilde{V}_m(q) \mathbf{I}[Q_m^i=q]$, the
optimal control policy is given by
\begin{align}
&\Pi^* (\mathbf{Q}^i) = \arg\min_{\Pi(\mathbf{Q}^i)}
\Big\{\overline{g}\big(\mathbf{Q}^i,\Pi(\mathbf{Q}^i),\gamma\big) +
    \sum_{\mathbf{Q}_j}\Pr[\mathbf{Q}^j|\mathbf{Q}^i,\Pi(\mathbf{Q}^i)] V(\mathbf{Q}^j)\Big\}\nonumber \\
=& \arg\Big\{   \sum_{m=S}^M Q_m^i + \gamma_{S,d} \mathbf{I}[Q_S^i =
N_Q] +
\sum_n f_X(n) V(\mathbf{Q}^i_{S,n})\nonumber\\
&+\min_{\Pi(\mathbf{Q}^i)} \mathbf{E}_{\mathbf{H}} \Big[
\sum_{m,{N_{SR}}} I_{S,m}^{N_{SR}}G_{S,m}(N_{SR},p_{S,m}) +
\sum_{n,{N_{RD}}} I_{n,D}^{N_{RD}} G_{n,D}(N_{RD},p_{n,D})
    \Big]\Big\} \nonumber\\
    \Leftrightarrow& \arg\min_{\Pi(\mathbf{Q}^i)} \mathbf{E}_{\mathbf{H}} \Big[
\sum_{m,{N_{SR}}} I_{S,m}^{N_{SR}}G_{S,m}(N_{SR},p_{S,m}) +
\sum_{n,{N_{RD}}} I_{n,D}^{N_{RD}} G_{n,D}(N_{RD},p_{n,D})
    \Big]\Big\} \label{eqn:prob-arg}
\end{align}
where $\mathbf{Q}^i_{S,n} = [Q_S^i+n, Q_1^i, Q_2^i, ..., Q_M^i ]$
and $G_{S,m}(N_{SR},p_{S,m})=\gamma_{S,p} p_{S,m} + \sum_n f_X(n)
\Big( \widetilde{V}_S\big(Q_S^i -
    R_{S,m}(N_{SR},p_{S,m}) + n\big) -  \widetilde{V}_S (Q_S^i+n) \Big)  + \widetilde{V}_m\big(Q_m^i
    + R_{S,m}(N_{SR},p_{S,m})\big) -
    \widetilde{V}_m\big(Q_m^i\big)$ and $G_{n,D}(N_{RD},p_{n,D})=\gamma_{n,p} p_{n,D} +
\widetilde{V}_n\big(Q_n^i - R_{n,D}(N_{RD},p_{n,D})\big) -
    \widetilde{V}_n(Q_n^i)$.
\label{prob:dist}
\end{Problem}
The solution of Problem \ref{prob:dist} is summarized in Lemma
\ref{lem:decom} below.
\begin{Lemma}[Distributive Control Policy]
Given the per-node value functions $\{\widetilde{V}_m(q)\}$
($\forall m=S,1,...,M$) and any realization of CSI $\mathbf{H}$ and
QSI $\mathbf{Q}^i$\footnote{Note that the following expressions are
all functions of the systems state. We omit the system state for
notation simplicity when the meaning is clear.}, the following
distributive control solves the Problem \ref{prob:dist}:
\begin{itemize}
\item Power control for the S-R link and R-D link ($\forall m=1,\cdots,M$):
\begin{eqnarray}
p_{S,m}^*(N_{SR}) =\arg \min_{p_{S,m}} G_{S,m}(N_{SR}, p_{S,m})
\quad \mbox{and} \quad p_{m,D}^*(N_{RD})= \arg \min_{p_{m,D}}
G_{n,D}(N_{RD},p_{n,D})\label{eqn:optp}
\end{eqnarray}
where $N_{SR},N_{RD}=0,1,...,\min(N_T,N_R)$.

\item First-stage bid at RSs ($\forall m=1,\cdots,M$):
\begin{eqnarray}
A_m^*(N_{SR}) =
G_{S,m}\big(N_{SR},p_{S,m}^*(N_{SR})\big)\label{eqn:first-B}
\end{eqnarray}
where $N_{SR}=0,1,...,\min(N_T,N_R)$.

\item Second-stage bid at RSs ($\forall n=1,\cdots,M$):
\begin{align}
\big(I_n,N_{SR,n}\big)=&\arg \min_{(m,{N_{SR})}} \Big\{
A_m^*(N_{SR}) + G_{n,D}\big(N_{RD},p_{n,D}^*(N_{RD})\big) \Big\}\nonumber\\
 B_n^* 
=& G_{S,I_n}\big(N_{SR,n},p_{S,m}^*(N_{SR,n})\big)+
G_{n,D}\big(N_{RD,n},p_{n,D}^*(N_{RD,n})\big)\label{eqn:second-B}
\end{align}
where $N_{RD} = \min (N_T, N_R - N_{SR})$.
\end{itemize}
In addition, for sufficiently large source arrival rate $\lambda_S$,
$\frac{N_Q}{\lambda_S}$ and the average transmit power constraints
$\{P_S,P_R\}$, the power control policy in \eqref{eqn:optp} has the
following closed-form expression:
\begin{equation}
p_{S,m}^*(N_{SR}) = \frac{ {N_{SR}} \Big[ \widetilde{V}_S^{'}(Q_S^i)
- \widetilde{V}_m^{'}(Q_m^i) \Big] }{\gamma_{S,p} \ln 2} -
\sum_{j=1}^{N_{SR}} \frac{1}{\eta_{S,m}^j} \label{eqn:psm}
\end{equation}
\begin{equation}
p_{m,D}^*(N_{RD})= \frac{ {N_{RD}} \widetilde{V}_m^{'}(Q_m^i)
}{\gamma_{m,p} \ln 2} - \sum_{j=1}^{N_{RD}} \frac{1}{\eta_{m,D}^j}
\label{eqn:pmd},
\end{equation}
where $\widetilde{V}_S^{'}(Q_S^i) = \frac{\widetilde{V}_S(Q_S^i+1) -
\widetilde{V}_S(Q_S^i-1)}{2}$ and $\widetilde{V}_m^{'}(Q_m^i) =
\frac{\widetilde{V}_m(Q_m^i+1) - \widetilde{V}_m(Q_m^i-1)}{2}$.
\label{lem:decom}
\end{Lemma}

\begin{proof}\textit{ Please refer to Appendix C.
}
\end{proof}

\begin{Remark}[Multi-level Water-Filling Structure of the Control Policy]
The power control policy (\ref{eqn:psm}) and (\ref{eqn:pmd}) as well
as the RS selection and data stream allocation control policy in
(\ref{eqn:first-B}) and (\ref{eqn:second-B}) are functions of both
the CSI and QSI where they depend on the QSI indirectly via the
per-node value functions  $\{\widetilde{V}_m(q)\}$ ($\forall
m=S,1,\cdots,M$). The power control solution has the form of
multi-level water-filling where the power is allocated according to
the CSI while the water-level is adaptive to the QSI.
\end{Remark}

\subsection{Online Distributive Stochastic Learning Algorithm to Estimate the Per-node Value Functions $\{\widetilde{V}_m(q)\}$  and  the  LMs
$\{\gamma_{S,d},\gamma_{S,p},\gamma_{m,p} \}$}
\label{subsec:learning}

In Lemma \ref{lem:decom}, the control actions are functions of
per-node value functions $\{\widetilde{V}_m(q)\}$  and the LMs
$\{\gamma_{S,d},\gamma_{S,p},\gamma_{m,p} \}$. In this section, we
propose an online learning algorithm to determine the per-node value
functions and the LMs realtime. The almost-sure convergence proof of
this algorithm is provided in the next section. The system procedure
of the proposed
distributive online learning is given below.
\begin{itemize}
\item {\bf Step 1 [Initialization]:} Each RS $m$ initiates its per-node value functions and LMs, denoted as
$\{\widetilde{V}_m^0(q)\}$ and $\gamma_{m,p}^0$, as well as the
per-node value functions and LMs for the source node, denoted as
$\{\widetilde{V}_S^0(q)\}$ and $\{\gamma_{S,p}^0,\gamma_{S,d}^0\}$.
The initialization of $\widetilde{V}_S^0$ and
$\{\gamma_{S,p}^0,\gamma_{S,d}^0\}$ at each RS should be the same.

\item {\bf Step 2 [Determination of control actions]:}
At the beginning of the $t$-th frame, the source node broadcasts its
QSI $Q_S(t)$ to the RS nodes. Based on the local system information
$\big(Q_S(t),Q_m(t),\mathbf{H}_m(t)\big)$ and the per-node value
functions $\{\widetilde{V}^t_m(q)\}$ and $\{\widetilde{V}^t_S(q)\}$,
each RS $m$ determines the distributive control actions including
the S-R and R-D power allocation $p_{S,m}^*(N_{SR},t)$,
$p_{m,D}^*(N_{RD},t)$ the first-stage bid $A_m^*(N_{SR},t)$
($N_{SR}=1,\cdots, N_{\min}$) as well as the second-stage bid
$B_m^*(t)$, $I_n(t)$, $N_{SR,n}(t)$ according to Lemma
\ref{lem:decom}.  Based on the contention resolution protocol
described in Section \ref{subsec:model-prot}, the Rx-RS and the
Tx-RS pair is given by ($m^*(t),n^*(t)$) (where $n^*(t)=\arg\min_n
B^*_n(t)$ and $m^*(t)=I_{n^*(t)}(t)$) and the corresponding number
of data streams pair is given by $(N_{SR}^*(t), N_{RD}^*(t))$ (where
$N_{SR}^*(t)=N_{SR,n^*(t)}(t)$ and $N_{RD}^*(t)=N_{RD,n^*(t)}(t)$).


\item {\bf Step 3 [Per-node value functions and LMs update]:} Each
RS $m$ updates the per-node value function
$\{\widetilde{V}^{t+1}_S(q)\}$, $\{\widetilde{V}^{t+1}_m(q)\}$ as
well as the LMs $
\{\gamma_{S,d}^{t+1},\gamma_{S,p}^{t+1},\gamma_{m,p}^{t+1} \}$
according to Algorithm \ref{alg:update}. Finally, let $t=t+1$ and go
to Step 2.
\end{itemize}
\begin{Algorithm}[Online distributive learning algorithm for per-node value functions and LMs]
\begin{align}
\widetilde{V}_{m}^{t+1}(q)=& \widetilde{V}_m^{t} (q)+\epsilon^t_v
\Big[\gamma_{S,d}\mathbf I[Q_S(t)=N_Q] + q + B^*_{n^*(t)}(t)-
\widetilde{V}_m^t (q) \Big]\mathbf I[\mathbf{Q}(t) = \beta_{m,q}],
m=S,1,...,M \label{eqn:update-v}
\end{align}
\begin{align}
 \gamma_{S,d}^{t+1}  = & \Big(\gamma_{S,d}^{t} + \epsilon_{d}^t (\mathbf{I}[Q_S(t)=N_Q]-D)\Big)^+
\label{eqn:update-sd}\\
 \gamma_{S,p}^{t+1} =& \Big( \gamma_{S,p}^{t}
+ \epsilon_p^t \big (\sum_{N_{SR}=1}^{N_{\min}}
I_{S,m}^{N_{SR}}(t)p_{S,m}(N_{SR},t)-P_S
\big)\Big)^+\label{eqn:update-sp}\\
 \gamma_{m,p}^{t+1} =& \Big(
\gamma_{m,p}^{t} + \epsilon_p^t \big (\sum_{N_{SR}=1}^{N_{\min}}
I_{m,D}^{N_{RD}}(t)p_{m,D}(N_{RD},t)-P_R \big)\Big)^+ ,\ m=1,2,...,M
\label{eqn:update-rp}
\end{align}
where $I_{S,m}^{N_{SR}}(t)=\mathbf I [m=m^*(t)]\mathbf
I[N_{SR}=N_{SR}^*(t)]$, $I_{m,D}^{N_{RD}}(t)=\mathbf I
[m=n^*(t)]\mathbf I[N_{RD}=N_{RD}^*(t)]$, and
\textcolor{black}{$\{\epsilon^t_v>0\}$, $\{\epsilon^t_d>0\}$
$\{\epsilon_p^t>0\}$ are the step size sequences satisfying
\begin{align}
\sum_{t=0}^{\infty} \epsilon^t_{v} = \infty, \sum_{t=0}^{\infty}
\epsilon^t_p = \infty, \sum_{t=0}^{\infty} \epsilon^t_{d} = \infty,
\sum_{t=0}^{\infty}
\bigg[(\epsilon^t_{v})^2+(\epsilon^t_p)^2+(\epsilon^t_{d})^2\bigg]<\infty,
\lim_{t\rightarrow +\infty}
\frac{\epsilon^t_p}{\epsilon^t_{v}}=0,\lim_{t\rightarrow +\infty}
\frac{\epsilon^t_{d}}{\epsilon^t_v}=0.\nonumber
\end{align}}
\label{alg:update}
\end{Algorithm}

\subsection{Almost-Sure Convergence of Distributive Stochastic Learning}

In this section, we shall establish technical conditions for the
almost-sure convergence of the online distributive learning
algorithm. Since $\{\epsilon_v^t\}$, $\{\epsilon_p^t\}$,
$\{\epsilon_d^t\}$ satisfy $\epsilon_p^t=\mathbf{o}(\epsilon_v^t)$,
$\epsilon_{d}^t=\mathbf{o}(\epsilon_p^t)$, the LMs update and the
per-node potential functions update are done simultaneously but over
two different time scales. During the per-node potential functions
update (timescale I), we have $\gamma^{t+1}_p - \gamma^{t+1}_p =
\mathcal{O}(\epsilon_p^t) = \mathbf{o}(\epsilon_v^t)$ and
$\gamma^{t+1}_{S,d} - \gamma^{t+1}_{S,d} = \mathcal{O}(\epsilon_d^t)
= \mathbf{o}(\epsilon_v^t)$. Therefore, the LMs appear to be
quasi-static \cite{Borkar:08} during the per-node value function
update in (\ref{eqn:update-v}). For the notation convenience, define
the sequences of matrices $\{\mathbf{A}^t\}$ and $\{\mathbf{B}^t\}$
as $\mathbf{A}^{t-1}=(1-\epsilon_v^{t-1})\mathbf{I} +
\mathbf{M}^{-1} \mathbf{P}(\Pi^{t}) \mathbf{M} \epsilon_v^{t-1}$ and
$\mathbf{B}^{t-1}=(1-\epsilon_v^{t-1})\mathbf{I} + \mathbf{M}^{-1}
\mathbf{P}(\Pi^{t-1}) \mathbf{M} \epsilon_v^{t-1}$,
where $\Pi^t$ is a unichain system control policy at the $t$-th
frame, $\mathbf{P}(\Pi^{t})$ is the transition matrix of system
states given the unichain system control policy $\Pi^t$,
$\mathbf{I}$ is identity matrix. The convergence property of the
per-node value function update is given below:

\begin{Lemma}[Convergence of Per-Node Value Function Learning over Timescale I]
Assume for all the feasible policy in the policy space, there exists
some positive integer $\beta$ and $\tau^{\beta}>0$ such that
\begin{equation}
[\mathbf{A}^{\beta-1}...\mathbf{A}^1]_{(a,I)} \geq \tau^{\beta},
\quad [\mathbf{B}^{\beta-1}...\mathbf{B}^1]_{(a,I)} \geq
\tau^{\beta} \quad \forall a,\label{eqn:con_matrix}
\end{equation}
where $[\cdot]_{(a,I)}$ denotes the element in $a$-th row and $I$-th
column (where $I$ corresponds to the reference state $\mathbf Q^I$)
and $\tau^t=\mathcal{O}(\epsilon_v^t)$ ($\forall t$). The following
statements are true:
\begin{itemize}
\item The update of the parameter vector (or per-node potential vector) will converge almost surely for any given initial parameter vector $\mathbf{W}^0$ and LMs $\gamma$, i.e.
$ \lim\limits_{t\rightarrow \infty} \mathbf{W}^t(\gamma) =
\mathbf{W}^{\infty}(\gamma)$.

\item The steady state parameter vector $\mathbf{W}^{\infty}$ satisfies:
\begin{equation}
\theta \mathbf{e} + \mathbf{W}^{\infty}(\gamma) = \mathbf{M}^{-1}
\mathbf{T}\big(\gamma, \mathbf{M} \mathbf{W}^{\infty}(\gamma)\big)
\label{eqn:Tmap}
\end{equation}
where $\theta$ is a constant, $\mathbf{W}^{\infty}$ is given by
\begin{equation}
\mathbf{W}^{\infty} =  \big[
\widetilde{V}^{\infty}_S(0),...,\widetilde{V}^{\infty}_S(N_Q),\widetilde{V}^{\infty}_1(0),...,\widetilde{V}^{\infty}_1(N_Q),...,\widetilde{V}^{\infty}_S(N_Q),\widetilde{V}^{\infty}_M(0),...,\widetilde{V}^{\infty}_M(N_Q)
\big]^T, \nonumber
\end{equation}
and the mapping $\mathbf{T}$ is defined as $\mathbf{T}(\gamma,
\mathbf{V}) = \min_{\Pi}[\overline{\mathbf{g}}(\gamma,\Pi) +
\mathbf{P}(\Pi) \mathbf{V}]$.
\end{itemize}
\label{lem:conv-I}
\end{Lemma}

\begin{proof}
\textit{Please refer to Appendix D.}
\end{proof}
\begin{Remark}[Interpretation of the Sufficient Conditions in Lemma \ref{lem:conv-I}]
Note that $A^t$ and $B^t$ are related to the transition probability
of the reference states. Condition (\ref{eqn:con_matrix}) simply
means that there is one reference state accessible from all the
other reference states after some finite number of transition steps.
This is a very mild condition and will be satisfied in most of the
cases in practice.
\end{Remark}

Note that \refbrk{eqn:Tmap} is equivalent to the following Bellman
equation on the representative states $\mathcal S_R$:
\begin{equation}
\theta + \widetilde{V}^{\infty}_m(q) =\min_{\Pi(\beta_{m,q})}\Big\{
\overline{g}\big(\beta_{m,q},\Pi(\beta_{m,q}),\gamma_k\big) +
\sum_{\mathbf{Q}^j}
\Pr\big[\mathbf{Q}^j|\beta_{m,q},\Pi(\beta_{m,q})\big] \sum_{m=S}^M
\widetilde{V}^{\infty}_m(Q^j_m)\Big\}, \quad \forall \beta_{m,q}\in
\mathcal S_R \nonumber.
\end{equation}
Hence, Lemma \ref{lem:conv-I} basically guarantees the proposed
online learning algorithm will converge to the {\em best fit}
parameter vector (per-node potential) satisfying
\eqref{eqn:per-node}. On the other hand, since the ratio of step
sizes satisfies
$\frac{\epsilon_p^t}{\epsilon_{v}^t},\frac{\epsilon_d^t}{\epsilon_{v}^t}\rightarrow
0$ during the LM update (timescale II), the per-node value function
will be updated much faster than the Lagrange multipliers. Hence,
the update of Lagrange multipliers in timescale II will trigger
another update process of the per-node value function in timescale
I. By the Corollary 2.1 of \cite{Borkar:97}, we have
$\lim\limits_{t\rightarrow \infty}
||\widetilde{\mathbf{V}}_m^t-\widetilde{\mathbf{V}}_m^{\infty}(\gamma^t)||=0$
w.p.1. Hence, during the LM updates in \refbrk{eqn:update-rp},
\refbrk{eqn:update-sp} and \refbrk{eqn:update-sd}, the per-node
value function update in \refbrk{eqn:update-v} is seen as almost
equilibrated. Moreover, convergence of the LMs is summarized below.

\begin{Lemma}[Convergence of the LMs over Timescale II]
The iteration on the vector of LMs
$\gamma=[\gamma_{S,d},\gamma_{S,p},\\
\gamma_{1,p},...,\gamma_{M,p}]^T$ converges almost surely to
$\gamma^*=[\gamma^*_{S,d},\gamma^*_{S,p},\gamma^*_{1,p},...,\gamma^*_{M,p}]^T$,
which satisfies the power and packet drop rate constraints in
\refbrk{eqn:s-p-cont},\refbrk{eqn:r-p-cont} and
\refbrk{eqn:d-cont}.\label{lem:conv-II}
\end{Lemma}

\begin{proof}
\textit{Please refer to Appendix E.}
\end{proof}

Based on the above lemmas, we summarized the convergence performance
of the online per-node value functions and LMs learning algorithm in
the following theorem.

\begin{Theorem}[Convergence of Online Learning Algorithm \ref{alg:update}]
For the same conditions as in Lemma \ref{lem:conv-I}, we have
$(\gamma^t,\mathbf{W}^t) \rightarrow
\big(\gamma^*,\mathbf{W}^{\infty}(\gamma^*)\big)$ w.p.1., where
$\big(\gamma^*,\mathbf{W}^{\infty}(\gamma^*)\big)$ satisfies $\theta
\mathbf{e} + \mathbf{W}^{\infty}(\gamma^*) = \mathbf{M}^{-1}
\mathbf{T}\big(\gamma^*, \mathbf{M}
\mathbf{W}^{\infty}(\gamma^*)\big)$ and the average power constraint
(\ref{eqn:s-p-cont},\ref{eqn:r-p-cont}) as well as the average
packet drop rate constraint \refbrk{eqn:d-cont}, where $\mathbf{e}$
is a $(M+1)(N_Q+1) \times 1$ vector with all elements equal to
1.\label{thm:con}
\end{Theorem}

\subsection{Asymptotic Optimality}

Finally, we shall show that the performance of the distributive
algorithm is asymptotically global optimal for high traffic loading.

\begin{Theorem}[Asymptotically Global Optimal at High Traffic Loading] For sufficiently large $N_Q$ and high traffic loading
such that the optimization problem in Problem \ref{prob:main} is
feasible, the performance of the proposed distributive control
algorithm is asymptotically global optimal. \label{thm:optimal}
\end{Theorem}

\begin{proof}
\textit{Please refer to Appendix F.}
\end{proof}

\section{Simulations and Discussions} \label{sec:sim}
\begin{figure}[t]
\centering
\includegraphics[height=5.5cm, width=8cm]{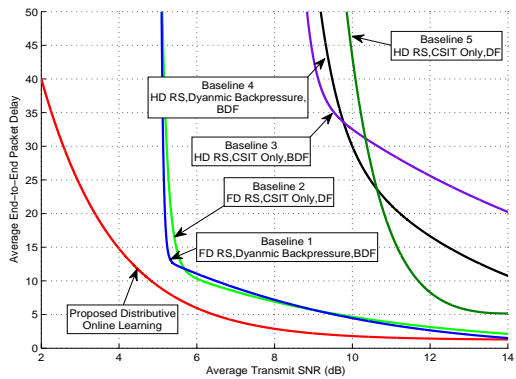}
\caption{Average end-to-end delay versus average transmit SNR.
Baseline 1 refers to the dynamic backpressure algorithm with BDF
protocol and full-duplex relays. Baseline 2 refers to the CSIT only
scheduling with traditional DF protocol and full-duplex relays.
Baseline 3 refers to the CSIT only scheduling with BDF protocol and
half-duplex relays. Baseline 4 refers to the dynamic backpressure
algorithm with BDF protocol and half-duplex relays. Baseline 5
refers to the CSIT only scheduling with traditional DF protocol and
half-duplex relays. The deterministic packet size is $N_b=25$K bits
and the number of antennas at each RS is $N_R=4$. The packet drop
rates of the Baselines 1-5 and the proposed distributive online
learning are 0.2\% 0.2\% 13\%, 3\%, 24\% and 0.2\% respectively.}
\label{Fig:snr}
\end{figure}

\begin{figure}[t]
\centering
\includegraphics[height=5.5cm, width=8cm]{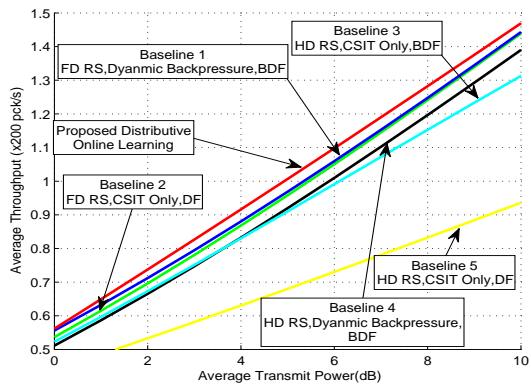}
\caption{Average throughput versus average transmit SNR. The
deterministic packet size is $N_b=30$K bits and the number of
antennas at each RS is $N_R=4$. The packet drop rates of the
Baselines 1-5 and the proposed distributive online learning are all
10\%.} \label{Fig:tp}
\end{figure}

\begin{figure}[t]
\centering
\includegraphics[height=6cm, width=8cm]{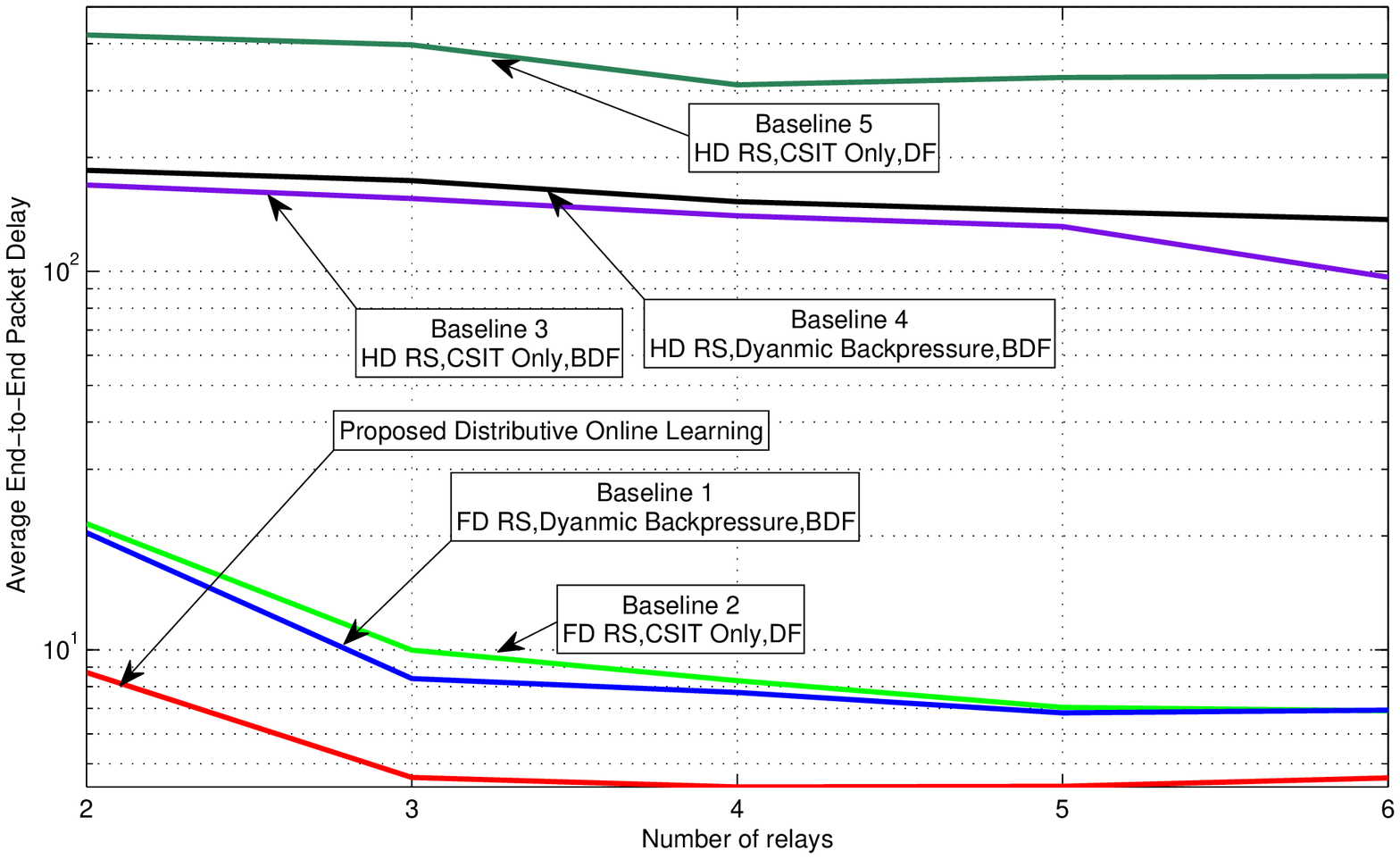}
\caption{Average end-to-end delay versus the number of relays with
transmit SNR = $5.5dB$. The deterministic packet size is $N_b=25$K
bits and the number of antennas at each RS is $N_R=4$. The packet
drop rates of the Baseline 1-5 and the proposed distributive online
learning are 23\%, 23\%, 86\%, 82\%, 96\% and 0.5\% respectively.}
\label{Fig:relay}
\end{figure}

\begin{figure}[t]
\centering
\includegraphics[height=6cm, width=8cm]{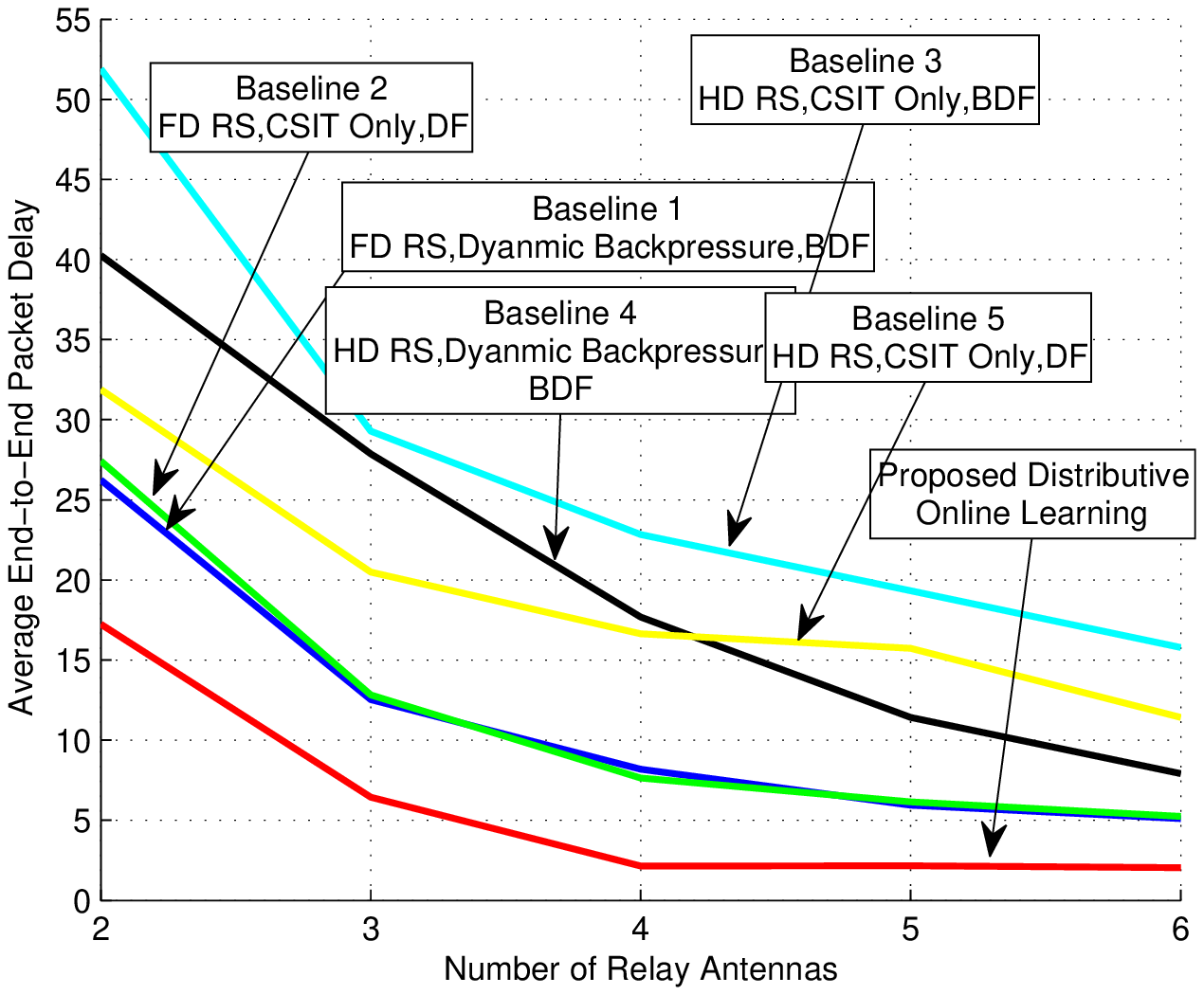}
\caption{Average end-to-end delay versus the number of relay
antennas with transmit SNR = $5dB$. The deterministic packet size is
$N_b=20$K bits and the number of antennas at each RS is $N_R=4$. The
packet drop rates of the Baseline 1-5 and the proposed distributive
online learning are 3\%, 4\%, 9\%, 5\%, 20\% and 0.1\%,
respectively.} \label{Fig:antenna}
\end{figure}

In this section, we shall compare our proposed online per-node value
function learning algorithm to five reference baselines. Baseline 1
and 4 refer to the proposed {\em buffered decode and forward} (BDF)
protocol with {\em throughput optimal policy} (in stability sense),
namely the {\it dynamic backpressure} algorithm
\cite{Georgiadis:book}, where we utilize full-duplex RSs in Baseline
1 and half-duplex RSs in Baseline 4. Baseline 2 and 5 refer to the
regular decode-and-forward protocol (DF) with the {\em CSIT only
scheduling} (the link selection and power allocation are adaptive to
the CSIT only so as to optimize the end-to-end throughput). We
utilize full-duplex RSs in Baseline 2 and half-duplex RSs in
Baseline 5. Moreover, Baseline 3 refers to the proposed BDF protocol
with CSIT only scheduling and half-duplex RSs. In the simulations,
we assume the total bandwidth is 1 MHz, the packet arrival at the
source node is Poisson with average arrival rate
$\lambda_S=200$pck/s and deterministic packet size $N_b$ bits. The
number of antennas at the source node and the destination node is
$N_T=2$. Moreover, the maximum buffer size of each node (source node
and RSs) is $N_Q=10$.

Figure \ref{Fig:snr} and \textcolor{black}{Figure \ref{Fig:tp}}
illustrate the average end-to-end delay and
\textcolor{black}{average throughput} versus average transmit SNR
per node with $N_R=4$ antennas at each RS, respectively. It can be
observed that the proposed distributive algorithm with half-duplex
RS could achieve significant performance gain \textcolor{black}{in
both average delay and average throughput} over all baselines with
full-duplex RSs, and even more significant gain over the baselines
with half-duplex RSs. This illustrates the advantages of the
proposed BDF algorithm with distributive delay-optimal control
policy, which could effectively reduce the intrinsic half-duplex
penalty in the cooperative communication systems.

Figure \ref{Fig:relay} and \textcolor{black}{Figure
\ref{Fig:antenna}} illustrate the average end-to-end delay versus
the number of relays and \textcolor{black}{the number of relay
antennas} with $N_R=4$ antennas at each RS, respectively.
\textcolor{black}{It can be observed that the average delay of all
the schemes decreases as the number of relays or the number of relay
antennas increases. Furthermore,} the proposed BDF algorithm with
distributive delay-optimal control policy has significant gain in
delay over all the baselines.

Figure \ref{Fig:conv} illustrates the convergence property of the
proposed distributive online learning algorithm. We plot the
per-node value function of the first relay versus scheduling slot
index at a transmit SNR= $10$dB. The average delay at the $200$-th
scheduling slot is already very close to the steady-state value,
which is much better than all the baselines. Furthermore, unlike the
iterations in deterministic NUM problems, the proposed algorithm is
online, meaning that normal payload is delivered during the
iteration steps.

\begin{figure}[h]
\centering
\includegraphics[height=6cm, width=8cm]{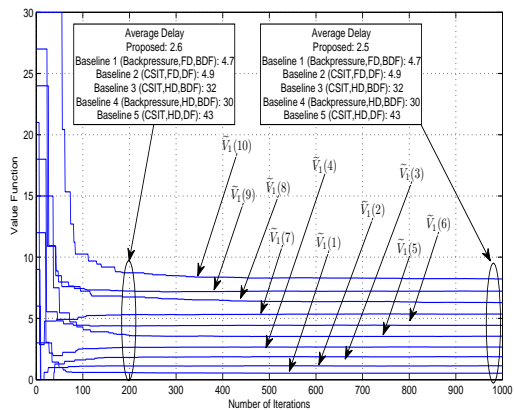}
\caption{Illustration of the convergence of the proposed online
learning algorithm. The instantaneous per-node value function is
plotted versus time slot index for a cooperative MIMO system with a
source node (with 2 antennas) and 2 RS nodes (each with 4 antennas).
The transmit SNR of the source and the RS nodes are 10 dB and the
target packet drop rate is 0.2\%. Unlike the iterations in
deterministic NUM problems, the proposed algorithm is online,
meaning that normal payload is delivered during the iteration
steps.} \label{Fig:conv}
\end{figure}

\section{Summary} \label{sec:sum}

In this paper, we consider queue-aware resource control for two-hop
cooperative MIMO systems. We show that by exploiting  buffering in
each MIMO relay, we could substantially reduce the intrinsic
half-duplex loss in cooperative systems. The  delay-optimal resource
control policy is formulated as an average-cost infinite horizon
Markov Decision Process (MDP). To obtain a low complexity solution,
we approximate the value function by a linear combination of
per-node value functions. The per-node value function is obtained
using a distributive stochastic learning algorithm. We also
established technical conditions for almost-sure convergence and
show that in heavy traffic limit, the proposed low complexity
distributive algorithm converges to global optimal solution.

\section*{Appendix A: Proof of Lemma \ref{lem:obj}}

The average number of bits received by the source node is given by
$\lambda_S (1-D)$, which is also the average number of information
bits received by the relay clusters as the source node and the relay
cluster are cascade. Let $W$, $W_S$ and $W_R$ be the average time
(with the unit of frames) one information bit staying in the system,
the source node's queue and some relay's queue respectively, $N_S$
and $N_R$ be the average number of information bits in the source
node's queue and the relays' queues respectively, we have $ N_S =
(1-D) \lambda_S W_S$ and $N_R = (1-D) \lambda_S W_R$ by Little's
Law. Notice that $W = W_S + W_R$, we have $ W = \frac{N_S +
N_R}{\lambda_S (1-D)}$. Since the change of system queue state forms
a Markov chain, we have $ W =
\mathbf{E}_{\pi_{\kappa}}\Big[\frac{Q_S + \sum_{m=1}^M
Q_m}{\lambda_S(1-D)}\Big]$, where $\pi_{\kappa}$ is the steady state
distribution. For sufficiently small packet drop rate requirement
$1-D \approx 1$, the end to end average delay becomes $W =
\mathbf{E}_{\pi_{\kappa}}\Big[\frac{Q_S + \sum_{m=1}^M
Q_m}{\lambda_S}\Big]$.

\section*{Appendix B: Proof of Lemma \ref{lem:reduce}}

From the Bellman equation of the original state space
(\ref{eqn:bellman}), we have
\begin{align}
\theta + V (\mathbf Q^i,\mathbf H)  = & \min_{\Pi(\mathbf
Q^i,\mathbf H)} \Big\{ g \big((\mathbf Q^i,\mathbf H),\Pi(\mathbf
Q^i,\mathbf H),\gamma\big) +
\sum_{(\mathbf Q^j,\mathbf H')} \Pr \big[(\mathbf Q^j,\mathbf H') | (\mathbf Q^i,\mathbf H), \Pi(\mathbf Q^i,\mathbf H) \big] J (\mathbf Q^j,\mathbf H') \Big\} \nonumber\\
\mathrel{\mathop=\limits^{(a)}}& \min_{\Pi(\mathbf Q^i,\mathbf H)}
\Big\{ g \big((\mathbf Q^i,\mathbf H),\Pi(\mathbf Q^i,\mathbf
H),\gamma\big) + \sum_{\mathbf{Q}^j} \Pr \big[ \mathbf{Q}^j |
(\mathbf Q^i,\mathbf H), \Pi(\mathbf Q^i,\mathbf H) \big] V
(\mathbf{Q}^j) \Big\} \label{eqn:org},
\end{align}
where (a) is due to the definition $V(\mathbf{Q}^j) =
\mathbf{E}_{\mathbf{H}'} [V(\mathbf{Q}^j,\mathbf{H}') |
\mathbf{Q}^j]$, and the optimal control actions are given by
$\Pi^*(\mathbf Q^i,\mathbf H) = \arg \min_{\Pi(\mathbf Q^i,\mathbf
H)} \Big\{ g \big((\mathbf Q^i,\mathbf H),\Pi(\mathbf Q^i,\mathbf
H),\gamma\big) + \sum_{\mathbf{Q}^j} \Pr \big[ \mathbf{Q}^j |
(\mathbf Q^i,\mathbf H), \Pi(\mathbf Q^i,\mathbf H) \big] V
(\mathbf{Q}^j) \Big\}$. Thus,
by the partitioning of the optimal control  actions in Definition
\ref{def:policy}, i.e. $ \Pi^*(\mathbf{Q}^i) = \{
\Pi^*(\mathbf{Q}^i, \mathbf H) | \forall \mathbf H \}$,
\begin{align}
\Pi^* (\mathbf{Q}^i)  =   \arg \min_{\Pi(\mathbf{Q}^i)}
\sum_{\mathbf{H}} \Pr (\mathbf{H}) \Big\{ g \big((\mathbf
Q^i,\mathbf H),\Pi(\mathbf Q^i,\mathbf H),\gamma\big) +
\sum_{\mathbf{Q}^j} \Pr \big[ \mathbf{Q}^j | (\mathbf Q^i,\mathbf
H), \Pi(\mathbf Q^i,\mathbf H) \big] V (\mathbf{Q}^j) \Big\}
\label{eqn:opt-pi}
\end{align}
From (\ref{eqn:org}) and (\ref{eqn:opt-pi}), we have $\theta + \Pr
(\mathbf{H}) V (\mathbf Q^i,\mathbf H) = \min_{ \Pi(\mathbf{Q}^i)}
\sum_{\mathbf{H}} \Pr (\mathbf{H}) \Big\{ g \big((\mathbf
Q^i,\mathbf H),\Pi(\mathbf Q^i,\mathbf H),\gamma\big) +
\sum_{\mathbf{Q}^j} \Pr \big[ \mathbf{Q}^j | (\mathbf Q^i,\mathbf
H), \Pi(\mathbf Q^i,\mathbf H) \big] V (\mathbf{Q}^j) \Big\}
\mathrel{\mathop=\limits^{(b)}}\min_{ \Pi(\mathbf{Q}^i)} \Big\{
\overline{g} \big(\mathbf{Q}^i,  \Pi(\mathbf Q^i),\gamma\big) +
\sum_{\mathbf{Q}^j} \Pr \big[ \mathbf{Q}^j | \mathbf{Q}^i,
 \Pi (\mathbf{Q}^i) \big] V (\mathbf{Q}^j) \Big\}$,
where the equality (b) is due to the definition of $\overline{g}$ in
(\ref{eqn:bar-g}). As a result, the control policy obtained by
solving (\ref{eqn:bellman})  is the same as that obtained by solving
(\ref{eqn:MDP})  and this completes the proof.

%

\section*{Appendix C: Proof of Lemma \ref{lem:decom}}

We shall prove the general control policy first, followed by the
closed-form power control derivation.

According to (\ref{eqn:prob-arg}), given $ N_{SR} $ and $ N_{RD} $,
the optimal power control is given by:
\begin{align}
&\min_{\{ \Pi_p^m \}} \mathbf{E}_{\mathbf{H}} \Big[
\sum_{m,{N_{SR}}} I_{S,m}^{N_{SR}}G_{S,m}(N_{SR},p_{S,m}) +
\sum_{n,{N_{RD}}} I_{n,D}^{N_{RD}} G_{n,D}(N_{RD},p_{n,D})
    \Big]\Big\} \nonumber\\
=& \mathbf{E}_{\mathbf{H}} \Big[ \sum_{m,{N_{SR}}} I_{S,m}^{N_{SR}}
\min_{p_{S,m}} G_{S,m}(N_{SR},p_{S,m}) + \sum_{n,{N_{RD}}}
I_{n,D}^{N_{RD}} \min_{p_{m,D}} G_{n,D}(N_{RD},p_{n,D})
    \Big]\Big\} \nonumber
\end{align}
Therefore, $p_{S,m}^*(N_{SR}) = \arg\min_{p_{S,m}}
G_{S,m}(N_{SR},p_{S,m})$ and $ p_{m,D}^*(N_{SR})=\arg\min_{p_{m,D}}
G_{n,D}(N_{RD},p_{n,D})$. To determine the optimal Rx-RS, Tx-RS and
stream allocation, the biding is divided into  two stages:
\begin{itemize}
\item First Biding: Each RS (say the $ m $-th RS) broadcasts one bid for each possible $ N_{SR} $ indicating that if itself is selected as Rx-RS and the number of S-R streams is $N_{SR}$, what would be the corresponding $ G_{S,m}(N_{SR},p_{S,m}^*) $.
\item Second Biding: After receiving the bids in the first round, each RS (say the $ n $-th RS) should calculate that if itself is selected as the Tx-RS, which RS else is the best Rx-RS (say the $ m $-th RS is the best Rx-RS), what's the best $ N_{SR} $ and $ N_{RD} $ and what's the corresponding $ B_n^*=G_{S,m}(N_{SR},p_{S,m}^*) + G_{n,D}(N_{RD},p_{n,D}^*)$. Then, broadcast the calculation results $ B_n^* $ as the second bid.
\item After comparing the $ B_n^* $, the optimal  Rx-RS, Tx-RS and stream allocation can be determined.
\end{itemize}
Therefore, the first-stage bidding and the second-stage bidding is
straight-forward.

When $\lambda_S$ and  $\frac{N_Q}{\lambda}$ ($m = S,1,2,...,M$) are
sufficiently large, it with large probability that
$\frac{Q_m}{\lambda}$ ($m = S,1,2,...,M$) is sufficiently large.
Hence, following a similar approach in \cite{Bettesh:06}, it can be
proved that the value function $\widetilde{V}_m$ ($m = S,1,2,...,M$)
is increasing polynomially in $Q = [Q_S,Q_1,...,Q_M]^T$. The
optimization on $ p_{S,m} $ is given by
\begin{align}
& p_{S,m}^*(N_{SR})= \arg\min_{p_{S,m}} G_{S,m}(N_{SR},p_{S,m}) \nonumber \\
=& \arg\min_{p_{S,m}} \bigg\{\gamma_{S,p} p_{S,m} + \sum_n f_X(n) \Big( \widetilde{V}_S\big(Q_S^i - R_{S,m}(N_{SR},p_{S,m}) + n\big) -  \widetilde{V}_S (Q_S^i+n) \Big) \nonumber\\
& + \widetilde{V}_m\big(Q_m^i + R_{S,m}(N_{SR},p_{S,m})\big)  -
\widetilde{V}_m\big(Q_m^i\big) \bigg\}.\label{eqn:temp-l3}
\end{align}
Similar to \cite{Bettesh:06}, we can do Taylor expansion as follows:
\begin{equation}
\widetilde{V}_S\big(Q_S^i - R_{S,m}(N_{SR},p_{S,m}) + n\big) =
\widetilde{V}_S\big(Q_S^i\big) + \bigg(n- R_{S,m}(N_{SR},p_{S,m})
\bigg)\widetilde{V}_S^{'}\big(Q_S^i \big), \label{eqn:temp-l3-1}
\end{equation}
\begin{equation}
\widetilde{V}_S (Q_S^i+n) = \widetilde{V}_S\big(Q_S^i\big) + n
\widetilde{V}_S^{'}\big(Q_S^i\big) \label{eqn:temp-l3-2}
\end{equation}
where $ \widetilde{V}_S^{'} $ is the first order derivative on
$\widetilde{V}_S$ and the higher order is neglectable. Same apporach
can be used to expand $ \widetilde{V}_m\big(Q_m^i +
R_{S,m}(N_{SR},p_{S,m})\big) $ as $\widetilde{V}_m\big(Q_m^i +
R_{S,m}(N_{SR},p_{S,m})\big) = \widetilde{V}_m\big(Q_m^i\big) +
R_{S,m}(N_{SR},p_{S,m}) \widetilde{V}_m^{'}\big(Q_m^i\big)$.
At high SNR region, we have
\begin{equation}
\frac{\partial R_{S,m}(N_{SR},p_{S,m})}{\partial p_{S,m}} =
\frac{N}{\ln 2}\frac{1}{p_{S,m}+\sum_{j=1}^{N_{SR}}
\frac{1}{\eta_{S,m}^j}}. \label{eqn:temp-l3-3}
\end{equation}
According to
(\ref{eqn:temp-l3-1},\ref{eqn:temp-l3-2},\ref{eqn:temp-l3-3}),
taking derivative on the RHS of (\ref{eqn:temp-l3}) and letting it
be zero, we can get the closed-from expression for power allocation
in (\ref{eqn:psm}). Moreover, (\ref{eqn:pmd}) can be proved in the
same way. Finally, when $ Q_m $ and $ Q_S $ are sufficiently large,
according to the definition of derivative, we have
\begin{equation}
\widetilde{V}_S^{'}(Q_S^i) = \frac{\widetilde{V}_S(Q_S^i+1) -
\widetilde{V}_S(Q_S^i-1)}{2} \qquad \widetilde{V}_m^{'}(Q_m^i) =
\frac{\widetilde{V}_m(Q_m^i+1) - \widetilde{V}_m(Q_m^i-1)}{2}
\nonumber.
\end{equation}

\section*{Appendix D: Proof of Lemma \ref{lem:conv-I}}

From \cite{Borkar:98}, the convergence property of the asynchronous
update and synchronous update is the same. Therefore, we consider
the convergence of related synchronous version without loss of
generality.

Let $c\in R$ be a constant, we have $T_I(c \widetilde{V}_S^l) = c
T_I(\widetilde{V}_S^l)$, where $T_I$ is one element of mapping
$\mathbf{T}$ corresponding to the state with all buffers empty.
Similar to \cite{Borkar:00}, the per-node value function
$\{\widetilde{\mathbf{V}}_m\}$ is bounded almost surely during the
iterations of algorithm. According to the construction of parameter
vector $\mathbf{W}$, the update on $\widetilde{\mathbf{V}}_m$ is
equivalent to the update on $\mathbf{W}$ and proving the convergence
of Lemma \ref{lem:conv-I} is equivalent to proving the convergence
of update on $\mathbf{W}$. In the following, we first introduce and
prove the following lemma on the convergence of learning noise.

\begin{Lemma}Define $ \mathbf{q}^l = \mathbf{M}^{\dag} \Big[
\overline{\mathbf{g}}(\Pi_l) + \mathbf{P}(\Pi_l)\mathbf{M}
\mathbf{W}^l - \mathbf{M} \mathbf{W}^l - T_{I}(\mathbf{M}
\mathbf{W}^l) \mathbf{e}\Big]$, when the number of iterations $l
\geq j \rightarrow \infty$, the procedure of update can be written
as follows with probability 1: $\mathbf{W}^{l+1} = \mathbf{W}^{j} +
\sum_{i=j}^l \epsilon_v^i \mathbf{q}_m^i$. \label{lem:noise}
\end{Lemma}

The proof of above lemma follows the standard approach of stochastic
approximation with Martingale noise \cite{Borkar:08}. Moreover, the
following lemma is about the limit of sequence $\{\mathbf{q}_m^l\}$.

\begin{Lemma} Suppose the following two inequalities are true for $l=a,a+1,...,a+b$
\begin{align}
\overline{\mathbf{g}}(\Pi^{l}) + \mathbf{P}(\Pi^{l}) \mathbf{M}
\mathbf{W}^l \leq &
\overline{\mathbf{g}}(\Pi^{l-1}) + \mathbf{P}(\Pi^{l-1}) \mathbf{M} \mathbf{W}^l \label{eqn:l}\\
\overline{\mathbf{g}}(\Pi^{l-1}) + \mathbf{P}(\Pi^{l-1}) \mathbf{M}
\mathbf{W}^{l-1} \leq & \overline{\mathbf{g}}(\Pi^{l}) +
\mathbf{P}(\Pi^{l}) \mathbf{M} \mathbf{W}^{l-1} \label{eqn:l-1},
\end{align}
then we have
\begin{equation}
|q^{a+b}_i|\leq C_1 \prod_{i=0}^{\lfloor \frac{b}{\beta} \rfloor -
1}(1-\tau^{a+i \beta}) \quad \forall i, \label{eqn:con_cov}
\end{equation}
where $q_i^{a+b}$ denotes the $i$th element of the vector
$\mathbf{q}^{a+b}$, $C_1$ is some constant.\label{lem:iter}
\end{Lemma}

\begin{proof}
\textit{ From \refbrk{eqn:l} and \refbrk{eqn:l-1}, we have
\begin{equation}
\mathbf{q}^l =  \mathbf{M}^{\dag} \big[
\overline{\mathbf{g}}(\Pi^{l}) + \mathbf{P}(\Pi^{l})\mathbf{M}
\mathbf{W}^l - \mathbf{M} \mathbf{W}^l - w_l \mathbf{e}\big] \leq
\mathbf{M}^{\dag} \big[ \overline{\mathbf{g}}(\Pi^{l-1}) +
\mathbf{P}(\Pi^{l-1}) \mathbf{M} \mathbf{W}^l - \mathbf{M}
\mathbf{W}^l - w_l \mathbf{e}\big] \nonumber
\end{equation}
\begin{align}
\mathbf{q}^{l-1} =&  \mathbf{M}^{\dag} \big[
\overline{\mathbf{g}}(\Pi^{l-1}) + \mathbf{P}(\Pi^{l-1}) \mathbf{M}
\mathbf{W}^{l-1} - \mathbf{M} \mathbf{W}^{l-1} - w_{l-1} \mathbf{e}\big] \nonumber \\
\leq&  \mathbf{M}^{\dag} \big[ \overline{\mathbf{g}}(\Pi^{l}) +
\mathbf{P}(\Pi^{l}) \mathbf{M} \mathbf{W}^{l-1} - \mathbf{M}
\mathbf{W}^{l-1} - w_{l-1} \mathbf{e}\big] \nonumber
\end{align}
where $w_l = T_{I}(\mathbf{M} \mathbf{W}^l)=T_{I}(\mathbf{M}
\mathbf{W}^l)$. According to Lemma \ref{lem:noise}, we have
$\mathbf{W}^l = \mathbf{W}^{l-1} + \epsilon_{v}^{l-1}
\mathbf{q}^{l-1} \Rightarrow \mathbf{W}^l = \mathbf{W}^{l-1} +
\epsilon_{v}^{l-1} \mathbf{q}^{l-1}$.
Therefore,
\begin{align}
\mathbf{q}^l \leq& \big[(1-\epsilon_v^{l-1})\mathbf{I} +
\mathbf{M}^{\dag} \mathbf{P}(\Pi^{l-1}) \mathbf{M} \epsilon_v^{l-1}
\big] \mathbf{q}^{l-1} + w_{l-1} \mathbf{e} - w_l \mathbf{e}
  = \mathbf{B}^{l-1} \mathbf{q}^{l-1} + w_{l-1} \mathbf{e} - w_l \mathbf{e} \nonumber\\
\mathbf{q}^l \geq & \big[(1-\epsilon_v^{l-1})\mathbf{I} +
\mathbf{M}^{\dag} \mathbf{P}(\Pi^{l}) \mathbf{M} \epsilon_v^{l-1}
\big] \mathbf{q}^{l-1} + w_{l-1} \mathbf{e} - w_l \mathbf{e} =
\mathbf{A}^{l-1} \mathbf{q}^{l-1} + w_{l-1} \mathbf{e} - w_l
\mathbf{e}.\nonumber
\end{align}
Notice that $\mathbf{A}^{l-1}\mathbf{e}=\mathbf{B}^{l-1}\mathbf{e}$,
we have $\mathbf{A}^{l-1}...\mathbf{A}^{l-\beta}
\mathbf{q}^{l-\beta} - C_1 \mathbf{e} \leq \mathbf{q}^l \leq
\mathbf{B}^{l-1}...\mathbf{B}^{l-\beta} \mathbf{q}^{l-\beta} - C_1
\mathbf{e}$
\begin{align}
\Rightarrow& (1-\tau^l) [\min \mathbf{q}^{l-\beta}]\leq \mathbf{q}^l
+ C_1 \mathbf{e} \leq (1-\tau^l) [\max \mathbf{q}^{l-\beta}]
\Rightarrow\begin{cases}
    \max \mathbf{q}^l + C_1 \leq (1-\tau^l) \max \mathbf{q}^{l-\beta}  \nonumber\\
    \min \mathbf{q}^l + C_1 \geq (1-\tau^l) \min \mathbf{q}^{l-\beta}\end{cases} \nonumber\\
\Rightarrow& \max \mathbf{q}^l - \min \mathbf{q}^l \leq (1-\tau^l) [
\max \mathbf{q}^{l-\beta} - \min \mathbf{q}^{l-\beta}] \Rightarrow
|q^l_i| \leq \max \mathbf{q}^l - \min \mathbf{q}^l \leq C_2
(1-\tau^l) \quad \forall i \nonumber,
\end{align}
where the first step is due to conditions of Lemma \ref{lem:conv-I}
on matrix sequence $\{\mathbf{A}^l\}$ and $\{\mathbf{B}^l\}$, $\max
\mathbf{q}^l$ and $\min \mathbf{q}^l$ denote the maximum and minimum
elements in $\mathbf{q}^l$ respectively, $C_1$ and $C_2$ are all
constants, the first inequality of the last step is because $\min
\mathbf{q}^l \leq 0$. This completes the proof of Lemma
\ref{lem:iter}.}
\end{proof}

Therefore, the proof of Lemma \ref{lem:conv-I} can be divided into
the following steps: (1) From the property of sequence
$\{\epsilon_v^l\}$, we have $\prod_{i=0}^{\lfloor \frac{l}{\beta}
\rfloor - 1}(1-\epsilon_v^{i \beta}) \rightarrow 0$ ($l \rightarrow
\infty$). (2) According to the first step, note that
$\tau^l=\mathcal{O}(\epsilon_v^l)$, from \refbrk{eqn:con_cov}, we
have $\mathbf{q}^l \rightarrow 0$ ($l \rightarrow \infty$). (3)
Therefore, the update on $\{\mathbf{W}^l\}$ will converge, and the
fixed point of the convergence $\mathbf{W}^{\infty}$ satisfies
$T_{I}(\mathbf{M} \mathbf{W}^l)\mathbf{e} + \mathbf{W}^{\infty} =
\mathbf{M}^{\dag} \mathbf{T}(\mathbf{M}\mathbf{W}^{\infty})$.


\section*{Appendix E: Proof of Lemma \ref{lem:conv-II}}

Due to the page limit, we only provide the sketch of the proof. The
convergence proof of the LMs
$\{\gamma_{S,p},\gamma_{1,p},...,\gamma_{M,p}\}$ for a given
$\gamma_{S,d}$ is as follows:
\begin{itemize}
\item For the notation convenience, we first define the average transmit power of each node as
follows: $\widetilde{\mathcal{P}}_S (\gamma) = \mathbf{E}^{\Pi}
\Big[ \sum_{m=1}^M \sum_{i=1}^{\min(N_T,N_R)} \eta_{S,m}^i p_{S,m}^i
\Big] \quad \mbox{and} \quad \widetilde{\mathcal{P}}_m (\gamma)=
\mathbf{E}^{\Pi} \Big[ \sum_{m=1}^M \sum_{i=1}^{\min(N_T,N_R)}
\eta_{m,D}^i p_{m,D}^i \Big]$ ($m=1,2,...,M$),
where $\mathbf{E}^{\Pi}[\cdot]$ denotes the expectation w.r.t. the
policy $\Pi(\gamma)$. Using standard stochastic approximation
theory, the dynamics of the LMs update equation
$\{\gamma_{S,p},\gamma_{1,p},...,\gamma_{M,p}\}$ can be represented
by the following ODE:
\begin{equation}
\big[\dot{\gamma}_{S,p}(t),...,\dot{\gamma}_{M,p}(t) \big]^T=
\big[\widetilde{\mathcal{P}}_S(\gamma) - P_S,
\widetilde{\mathcal{P}}_1(\gamma) - P_R, ...,
\widetilde{\mathcal{P}}_M(\gamma) - P_R \big]^T \label{eqn:ODE}.
\end{equation}

\item Using perturbation analysis in \cite{CaoXiRen:07}, we have $\frac{\partial \widetilde{\mathcal{P}}_m(\gamma)}{\partial
\gamma_{m,p}} < 0$ ($m=S,1,2,...,M$) and $\Big| \frac{\partial
\widetilde{\mathcal{P}}_m(\gamma)}{\partial \gamma_{m,p}} \Big| >>
\Big| \frac{\partial \widetilde{\mathcal{P}}_m(\gamma)}{\partial
\gamma_{n,p}} \Big|$ ($m=S,1,2,...,M , n\neq m$).
Thus, the update of $\gamma_{m,p}$ ($m=S,1,...,M$) in ODE
(\ref{eqn:ODE}) will drive $\widetilde{\mathcal{P}}_m-P_R$ (or
$\widetilde{\mathcal{P}}_S-P_S$) to 0 whenever
$\widetilde{\mathcal{P}}_m-P_R$ (or $\widetilde{\mathcal{P}}_S-P_S$)
is non-zero. Therefore, the ODE (\ref{eqn:ODE}) will converge. The
converged LMs
$\{\gamma_{S,p}^*(\gamma_{S,d}),\gamma_{1,p}^*(\gamma_{S,d}),...,\gamma_{M,p}^*(\gamma_{S,d})\}$
can be characterized by the equilibrium point of the ODE
(\ref{eqn:ODE}), which is given by the RHS of \refbrk{eqn:ODE}
$\rightarrow 0$.
\end{itemize}
Suppose for a given $\gamma_{S,d}$,
$\{\gamma_{S,p},\gamma_{1,p},...,\gamma_{M,p}\}$ converge to
$\{\gamma_{S,p}^*(\gamma_{S,d}),\gamma_{1,p}^*(\gamma_{S,d}),...,\gamma_{M,p}^*(\gamma_{S,d})\}$.
Since $\frac{\partial
\big(\mathbf{E}^{\Pi}_{\gamma_{1,p}^*,...,\gamma_{M,p}^*,\gamma_{S,p}^*}[Q_S=N_Q]\big)}{\partial
\gamma_{S,d}} < 0$,
the update on $\gamma_{S,d}$ will converge as well for a similar
reason as in the convergence of
$\{\gamma_{S,p},\gamma_{1,p},...,\gamma_{M,p}\}$. Similarly, the
converged $\gamma_{S,d}^*$ can be characterized by the equilibrium
point of the ODE
$\dot{\gamma}_{S,d}(t)=\mathbf{E}^{\Pi}_{\gamma_{1,p}^*,...,\gamma_{M,p}^*,\gamma_{S,p}^*}[Q_S=N_Q]
- D$, which is given by the RHS $\rightarrow 0$.

\section*{Appendix F: Proof of Theorem \ref{thm:optimal}}

Without loss of generality, we shall consider the approximate value
function $V(\mathbf{Q}) = \sum_{m=S}^M \sum_{q=1}^{N_Q}
\widetilde{V}_m(q)$ $ \mathbf{I}[Q_m=q]$
on the following redefined set of representative states
$\mathcal{Q}_R = \{\delta_{m,q}|m=S,1,2,...,M;
q=0,1,...,q_I-1,q_I+1,...,N_Q\}$, where the state $\delta_{m,q}$ is
given by $\delta_{m,q} = [Q_S=q_I, Q_1=q_I, ..., Q_m = q, ...,
Q_M=q_I]^T$
and $q_I < N_Q$ is sufficiently large. Correspondingly,
$\mathbf{M}^{-1}$ should also be redefined such that the per-node
value function $\{\widetilde{\mathbf{V}}_m\}$ is updated on the
representative states $\mathcal{Q}_R$ \cite{Tsitsiklis:96}.

First of all, following the similar approach in the proof of Lemma
\ref{lem:conv-I}, the per-node value function (under the new
reference states) would also converge almost surely to
$\{\widetilde{\mathbf{V}}_m^{\infty}(\gamma)\}$ for any given LMs
$\gamma$.

Next, when the conditions of Theorem \ref{thm:optimal} are
satisfied, given any $\epsilon>0$, there is one integer
$Q_0(\epsilon)$ such that for all $q>Q_0(\epsilon)$ and
$q_I=Q_0(\epsilon)$, we have (from the proof of Lemma
\ref{lem:decom}):
\begin{equation}
\widetilde{V}_m^{\infty}(q-r) - \widetilde{V}_m^{\infty}(q) =
\widetilde{V}_m^{\infty}(q_I-r) - \widetilde{V}_m^{\infty}(q_I) +
\mathcal{O}(\epsilon)\label{eqn:app-linear-1}.
\end{equation}
Moreover, since $\{\widetilde{V}_m^{\infty}(q)\}$ are all
monotonically increasing functions with respect to $q$ and
$\{\widetilde{V}_m^{\infty}(N_Q)\}$ are all
bounded\footnote{$\widetilde{V}_m^{\infty}(q)$ measures the
contribution of cost $q+\gamma_{m,p}\sum_k p_{m,D,k}$ if the system
starts at $Q=q$. For finite $N_Q$, $\widetilde{V}^{\infty}(q)$ is
always bounded. On the other hand, since the system is stable, the
queue length is bounded with probability 1 for arbitrarily large
$N_Q$ and hence, $\widetilde{V}_m^{\infty}(N_Q)$ must be bounded
almost surely for arbitrarily large $N_Q$.}, we have
$\widetilde{V}_m\Big(Q_0(\epsilon)\Big)=\mathcal{O}(\epsilon)$ for
sufficiently large arrivals. Therefore, \refbrk{eqn:app-linear-1}
holds for all $q\in[0,N_Q]$ for sufficiently large $N_Q$ and input
arrivals. Similarly, we have
\begin{equation}
\widetilde{V}_S^{\infty}(q+n-r) - \widetilde{V}_S^{\infty}(q+n) =
\widetilde{V}_S^{\infty}(q_I+n-r) - \widetilde{V}_S^{\infty}(q_I+n)
+ \mathcal{O}(\epsilon)\label{eqn:app-linear-2}
\end{equation}
\begin{equation}
\widetilde{V}_m^{\infty}(q+r) - \widetilde{V}_m^{\infty}(q) =
\widetilde{V}_m^{\infty}(q_I+r) - \widetilde{V}_m^{\infty}(q_I) +
\mathcal{O}(\epsilon).\label{eqn:app-linear-3}
\end{equation}

Hence, with the above equations and substituting the converged
per-node value function
$\{\widetilde{\mathbf{V}}_m^{\infty}(\gamma)\}$ into
(\ref{eqn:bellman}) for the reference states, we get
\begin{align}
\widetilde{V}_S^{\infty}(q)=& q + \gamma_{S,d}\mathbf{I}[q=N_Q] +
\sum_n f_X(n) \Big( \widetilde{V}_S^{\infty}(q+n) -
\widetilde{V}_S^{\infty}(n)\Big) + \min_{\Pi}
\mathbf{E}_{\mathbf{H}}\Big\{\sum_{m,N_{SR}} \eta_{S,m}^{N_{SR}}
\Big[ \gamma_{S,p} \sum_k
p_{S,m}^{N_{SR}}\nonumber\\
& + \sum_n f_X(n) \Big(
\widetilde{V}_S^{\infty}(q+n-r_{S,m}^{N_{SR}}) -
\widetilde{V}_S^{\infty}(q+n) \Big)  +
\widetilde{V}_m^{\infty}(q_I+r_{S,m}^{N_{SR}}) -
\widetilde{V}_m^{\infty}(q_I) \Big] \Big\} \label{eqn:vs-app}\\
\widetilde{V}_m^{\infty}(q) = & q + \widetilde{V}_m^{\infty}(q) +
\min_{u} \mathbf{E}_{\mathbf{H}} \Big\{ \sum_{N_{RD}}
\eta_{m,D}^{N_{RD}} \Big[ \gamma_{m,p} \sum_{k} p_{m,D}^{N_{RD}} +
\widetilde{V}_m^{\infty}(q-r_{m,D}^{N_{RD}}) -
\widetilde{V}_m^{\infty}(q) \Big] \Big\}, \label{eqn:vm-app}
\end{align}
where $m=1,2,...,M$.

Finally, for any system state $\mathbf{Q}^i=[Q_S^i,...,Q_M^i]^T$,
substitute the above equations into the RHS of the original Bellman
equation in (\ref{eqn:bellman}), we get RHS of (\ref{eqn:bellman})
$\mathrel{\mathop=\limits^{\rm a}}\sum_{m=S}^M Q_m^i +
\gamma_{S,d}\mathbf{I}[Q_S^i=N_Q] +
 \sum_{n}f_{X}(n) \widetilde{V}_S^{\infty}(Q_S^i+n) + \sum_{m=1}^M
\widetilde{V}_m^{\infty}(Q_m^i) +
\min_{\Pi(\mathbf{Q}^i)}\mathbf{E}_{\mathbf{H}}\Big\{
\sum_{m,N_{SR}} \eta^{N_{SR}}_{S,m}\Big[ \gamma_{S,p}
p_{S,m}^{N_{SR}} + \sum_{n}f_{X}(n)
  \Big(\widetilde{V}^{\infty}_S(Q_S^i+n-r_{S,m}^{N_{SR}})-\widetilde{V}^{\infty}_S(Q_S^i+n)\Big) + \widetilde{V}^{\infty}_m(q_I+r_{S,m}^{N_{SR}})
-\widetilde{V}^{\infty}_m(q_I)\Big]+ \sum_{m,N_{RD}}
\eta^{N_{RD}}_{m,D}\Big[\gamma_{m,p} p_{m,D}^{N_{RD}} +
\widetilde{V}^{\infty}_m(Q_m^i-r_{m,D}^{N_{RD}}) -
\widetilde{V}^{\infty}_m(Q_m^i) \Big] \Big\} + \mathcal{O}(\epsilon)
 \mathrel{\mathop=\limits^{\rm b}}  \sum_{m=S}^M \widetilde{V}_m^{\infty}(Q_m^i) + \sum_{n} f_I(n)
\widetilde{V}_S^{\infty}(n) + \mathcal{O}(\epsilon) =
V(\mathbf{Q}^i) + \sum_{n} f_I(n) \widetilde{V}_S^{\infty}(n) +
\mathcal{O}(\epsilon)$,
where equality (a) is due to (\ref{eqn:app-linear-3}), equality (b)
is due to (\ref{eqn:vs-app}) and (\ref{eqn:vm-app}). Since $\sum_{n}
f_X(n) \widetilde{V}_S^{\infty}(n)$ is a constant independent of
$\mathbf{Q}^i$ and $\epsilon$ is chosen arbitrarily, we have shown
that the approximate value function $V(\mathbf{Q}) = \sum_{m=S}^M
\sum_{q=1}^{N_Q} \widetilde{V}^{\infty}_m(q) \mathbf{I}[Q_m=q]$ can
satisfy the original Bellman equation \refbrk{eqn:bellman}
asymptotically (when $N_Q \rightarrow +\infty$). As a result, the
proposed distributive update algorithm converges to the global
optimal solution and this completes the proof.

\bibliographystyle{IEEEtran}
\bibliography{IEEEfull,ray,led,yan,yanchen_bibfile}

\begin{thebibliography}{10}
\providecommand{\url}[1]{#1}
\def\UrlFont{\rmfamily}
\providecommand{\newblock}{\relax}
\providecommand{\bibinfo}[2]{#2}
\providecommand\BIBentrySTDinterwordspacing{\spaceskip=0pt\relax}
\providecommand\BIBentryALTinterwordstretchfactor{4}
\providecommand\BIBentryALTinterwordspacing{\spaceskip=\fontdimen2\font plus
\BIBentryALTinterwordstretchfactor\fontdimen3\font minus
  \fontdimen4\font\relax}
\providecommand\BIBforeignlanguage[2]{{%
\expandafter\ifx\csname l@#1\endcsname\relax
\typeout{** WARNING: IEEEtran.bst: No hyphenation pattern has been}%
\typeout{** loaded for the language `#1'. Using the pattern for}%
\typeout{** the default language instead.}%
\else
\language=\csname l@#1\endcsname
\fi
#2}}

\bibitem{Meulen:68}
E.~van~der Meulen, ``Transmission of information in a t-terminal discrete
  memoryless channel,'' Ph.D. dissertation, Dep. of Statistics, University of
  California, Berkeley, 1968.

\bibitem{Cover:79}
T.~Cover and A.~Gamal, ``Capacity theorems for the relay channel,''
  \emph{{IEEE} Transactions on Information Theory}, vol.~25, no.~5, pp.
  572--584, Sep 1979.

\bibitem{WiMaxRelay:site}
IEEE 802.16's Relay Task Group. [Online]. Available:
  http://www.ieee802.org/16/relay/index.html.

\bibitem{WINNER:site}
WINNER- Wireless World Initiative New Radio. [Online]. Available:
  http://www.ist-winner.org/.

\bibitem{Yip:87}
P.~Yip and D.~Etter, ``An adaptive technique for multiple echo cancelation in
  telephone networks,'' in \emph{Acoustics, Speech, and Signal Processing, IEEE
  International Conference on ICASSP '87.}, vol.~12, Apr 1987, pp. 2133--2136.

\bibitem{Vega:08}
L.~Vega, H.~Rey, J.~Benesty, and S.~Tressens, ``A new robust variable step-size
  nlms algorithm,'' \emph{{IEEE} Transactions on Signal Processing}, vol.~56,
  no.~5, pp. 1878--1893, May 2008.

\bibitem{ErnestLo:07}
E.~Lo and K.~Letaief, ``Optimizing downlink throughput with user cooperation
  and scheduling in adaptive cellular networks,'' in \emph{Wireless
  Communications and Networking Conference, 2007.WCNC 2007. IEEE}, March 2007,
  pp. 4342--4347.

\bibitem{Bertsekas:1987}
D.~P. Bertsekas, \emph{Dynamic Programming - Deterministic and Stochastic
  Models}.\hskip 1em plus 0.5em minus 0.4em\relax Prentice Hall, NJ, USA, 1987.

\bibitem{Cao:2008}
X.~Cao, \emph{Stochastic Learning and Optimization: A Sensitivity-Based
  Approach}.\hskip 1em plus 0.5em minus 0.4em\relax Springer, 2008.

\bibitem{Yeh:05}
E.~Yeh and R.~Berry, ``Throughput optimal control of cooperative relay
  networks,'' in \emph{Information Theory, 2005. ISIT 2005. Proceedings.
  International Symposium on}, Sept. 2005, pp. 1206--1210.

\bibitem{Georgiadis-Neely-Tassiulas:2006}
L.~Georgiadis, M.~J. Neely, and L.~Tassiulas, ``Resource allocation and
  cross-layer control in wireless networks,'' \emph{Foundations and Trends in
  Networking}, vol.~1, no.~1, pp. 1--144, 2006.

\bibitem{Abounadi:98}
J.~Abounadi, D.~Bertsekas, and V.~S. Borkar, ``Learning algorithms for markov
  decision processes with average cost,'' \emph{SIAM Journal on Control and
  Optimization}, vol.~40, pp. 681--698, 1998.

\bibitem{Richardson:01}
T.~Richardson and R.~Urbanke, ``The capacity of low-density parity-check codes
  under message-passing decoding,'' \emph{{IEEE} Transactions on Information
  Theory}, vol.~47, no.~2, pp. 599--618, Feb 2001.

\bibitem{Boyd:04}
S.~Boyd and L.~Vandenberghe, \emph{Convex Optimization}.\hskip 1em plus 0.5em
  minus 0.4em\relax Cambridge University Press, 2004.

\bibitem{HanZhu:07}
Z.~Han, Z.~Ji, and K.~Liu, ``Non-cooperative resource competition game by
  virtual referee in multi-cell ofdma networks,'' \emph{{IEEE} Journal on
  Selected Areas in Communications}, vol.~25, no.~6, pp. 1079--1090, August
  2007.

\bibitem{Palomar:07}
D.~Palomar and M.~Chiang, ``Alternative distributed algorithms for network
  utility maximization: Framework and applications,'' \emph{{IEEE} Transactions
  on Automatic Control}, vol.~52, no.~12, pp. 2254--2269, Dec. 2007.

\bibitem{Huangjianwei:06}
J.~Huang, R.~Berry, and M.~L. Honig, ``Auction-based spectrum sharing,''
  \emph{ACM Mobile Networks and Applications Journal (MONET)}, vol.~11, pp.
  405--418, June 2006.

\bibitem{CurescuBidding:08}
C.~Curescu and S.~Nadjm-Tehrani, ``A bidding algorithm for optimized
  utility-based resource allocation in ad hoc networks,'' \emph{{IEEE}
  Transactions on Mobile Computing}, vol.~7, no.~12, pp. 1397--1414, Dec. 2008.

\bibitem{Bertsekas:07}
D.~Bertsekas, \emph{Dynamic Programming and Optimal Control, Vol. 2}.\hskip 1em
  plus 0.5em minus 0.4em\relax Athena Scientific, 2007.

\bibitem{Borkaractorcritic:2005}
V.{S.Borkar}, ``An actor-critic algorithm for constrained markov decision
  processes,'' in \emph{Systems Control Lett. 54}, 2005, pp. 207--213.

\bibitem{Tsitsiklis:96}
J.~N. Tsitsiklis and B.~van Roy, ``Feature-based methods for large scale
  dynamic programming,'' \emph{Machine Learning}, vol.~22, pp. 59--94, March
  1996.

\bibitem{Borkar:08}
V.~S. Borkar, \emph{Stochastic Approximation: A Dynamical Systems
  Viewpoint}.\hskip 1em plus 0.5em minus 0.4em\relax Cambridge University
  Press, 2008.

\bibitem{Borkar:97}
------, ``Stochastic approximation with two time scales,'' \emph{Systems
  Control Lett. 29}, pp. 291--294, 1997.

\bibitem{Georgiadis:book}
L.~Georgiadis, M.~Neely, and L.~Tassiulas, \emph{Resource Allocation and Cross
  Layer Control in Wireless Networks}.\hskip 1em plus 0.5em minus 0.4em\relax
  Now Publishers Inc, 2006.

\bibitem{Bettesh:06}
I.~Bettesh and S.~Shamai, ``Optimal power and rate control for minimal average
  delay: The single-user case,'' \emph{{IEEE} Transactions on Information
  Theory}, vol.~52, no.~9, pp. 4115--4141, Sept. 2006.

\bibitem{Borkar:98}
V.~S. Borkar, ``Asynchronous stochastic approximation,'' \emph{SIAM J. Control
  and Optim.}, vol.~36, pp. 840--851, 1998.

\bibitem{Borkar:00}
V.~S. Borkar and S.~P. Meyn, ``The ode method for convergence of stochastic
  approximation and reinforcement learning algorithms,'' \emph{SSIAM J. on
  Control and Optimization 38}, pp. 447--469, 2000.

\bibitem{CaoXiRen:07}
X.~Cao, \emph{Stochastic Learning and Optimization: A Sensitivity-Based
  Approach}.\hskip 1em plus 0.5em minus 0.4em\relax Springer, 2007.

\end{thebibliography}

\end{document}